\documentclass[12pt]{article}
\usepackage{amsmath}
\usepackage{graphicx}
\usepackage{natbib}
\usepackage{url} 
\usepackage{booktabs}
\usepackage{float}
\usepackage[ruled,vlined]{algorithm2e}

\usepackage{comment}

\usepackage{import}
\usepackage{amsmath, amssymb, amsfonts, amscd, xspace, pifont, amsthm, fancyhdr}
\usepackage{xcolor}
\usepackage{physics}
\usepackage{hyperref} 
\usepackage{mathrsfs}
\usepackage{dsfont}
\usepackage{caption} 
\usepackage{subcaption} 
\usepackage{import}
\usepackage{adjustbox}
\usepackage{mathtools}
\usepackage{lipsum}
\usepackage{dirtytalk}
\usepackage{bm}
\usepackage{tablefootnote}
\usepackage{threeparttable}
\usepackage{multirow}
\usepackage{multicol}
\usepackage{mathtools}
\usepackage[ruled,vlined]{algorithm2e}

\theoremstyle{definition}
\newtheorem{theorem}{Theorem}
\newtheorem*{theorem*}{Theorem}
\newtheorem{definition}{Definition}
\newtheorem{proposition}{Proposition}
\newtheorem*{proposition*}{Proposition}
\newtheorem{lemma}{Lemma}
\newtheorem*{lemma*}{Lemma}

\theoremstyle{remark}
\newtheorem{remark}{Remark}

\usepackage{enumitem}

\newcommand{\tcb}[1]{\textcolor{blue}{#1}}

\def\threeImages#1#2#3#4#5#6#7#8#9 
{
  \centerline{\hfill\makebox[#2]{#3}\hfill\makebox[#5]{#6}\hfill\makebox[#8]{#9}\hfill}
  \centerline{\hfill
    \includegraphics[width=#2]{#1}
    \hfill
    \includegraphics[width=#5]{#4}
    \hfill
    \includegraphics[width=#8]{#7}
    \hfill}
}

\def\twoImages#1#2#3#4#5#6 
{
  \centerline{\hfill\makebox[#2]{#3}\hfill\makebox[#5]{#6}\hfill}
  \centerline{\hfill
    \includegraphics[width=#2]{#1}
    \hfill
    \includegraphics[width=#5]{#4}
    \hfill}
}

\newcommand{\blind}{0}

\begin{document}

\def\spacingset#1{\renewcommand{\baselinestretch}%
{#1}\small\normalsize} \spacingset{1}



\if0\blind
{
  \title{\textbf{Variational Estimators for Node Popularity Models}}

  \author{
      Jony Karki\thanks{The first two authors made equal contributions.}
      \qquad
      Dongzhou Huang\footnotemark[1]
      \qquad
      Yunpeng Zhao
      \\[0.5em]
      Department of Statistics, Colorado State University
    }
    
  \maketitle
}
\fi

\if1\blind
{
  \bigskip
  \bigskip
  \bigskip
  \begin{center}
    {\LARGE\bf Title}
\end{center}
  \medskip
} \fi

\bigskip
\begin{abstract}
Node popularity is recognized as a key factor in modeling real-world networks, capturing heterogeneity in connectivity across communities. This concept is equally important in bipartite networks, where nodes in different partitions may exhibit varying popularity patterns, motivating models such as the Two-Way Node Popularity Model (TNPM). Existing methods, such as the Two-Stage Divided Cosine (TSDC) algorithm, provide a scalable estimation approach but may have limitations in terms of accuracy or applicability across different types of networks. In this paper, we develop a computationally efficient and theoretically justified variational expectation–maximization (VEM) framework for the TNPM. We establish label consistency for the estimated community assignments produced by the proposed variational estimator in bipartite networks. Through extensive simulation studies, we show that our method achieves superior estimation accuracy across a range of bipartite as well as undirected networks compared to existing algorithms. Finally, we evaluate our method on real-world bipartite and undirected networks, further demonstrating its practical effectiveness and robustness.
\end{abstract}

\noindent%
\vfill

\newpage
\spacingset{1.75} 
\section{Introduction}

Community detection is a fundamental problem in network analysis \citep{abbe2017community,zhao2017survey,Fortunato2010}, analogous to cluster analysis in unsupervised learning, but applied to the structure of adjacency matrices that represent relationships among entities. The stochastic block model (SBM) \citep{Holland83} and its variants are arguably the most extensively studied model-based framework for community detection \citep{Bickel&Chen2009,choi2012stochastic,amini2013pseudo,lei2015consistency, Airoldi2008, xu2013dynamic,matias2017statistical,Mariadassou2010weighted}. 

Community detection naturally extends to bipartite networks, which can be represented by a data matrix whose entries encode relationships between two distinct sets of entities, such as users and items in recommender systems.
The corresponding problem of biclustering, also referred to as co-clustering, on bipartite networks aims to simultaneously cluster the rows and columns of a rectangular adjacency matrix. Because such matrices are typically asymmetric, and in many cases even non-square, community detection methods developed for undirected networks cannot be directly applied to the bipartite setting.


For undirected networks, a major extension of the SBM is the degree-corrected stochastic block model (DCSBM) \citep{Karrer10}, which incorporates node-specific degree parameters to capture within-community degree heterogeneity. Building on this idea, \cite{sengupta2018block} proposed the popularity-adjusted block model (PABM), which allows node popularity parameters to vary across communities, rather than uniformly scaling each node’s connectivity as in the DCSBM. Similar developments have appeared for bipartite graphs. The latent block model (LBM), introduced by \cite{govaert2003clustering}, serves as the bipartite counterpart of the SBM, with each matrix entry modeled according to its row and column cluster labels \citep{mariadassou2015convergence}. Like the SBM, however, the LBM tends to cluster nodes according to degree patterns, ignoring other graph structures. To address this limitation, \cite{zhao2024variational} proposed a variational expectation–maximization (EM) algorithm for fitting the degree-corrected latent block model (DC-LBM).

Inspired by the PABM for undirected networks, \cite{jing2026two} recently proposed a more flexible model than the DC-LBM for directed and bipartite networks—the Two-Way Node Popularity Model (TNPM). Extending the popularity-adjustment idea to directed and bipartite networks, the TNPM introduces two sets of node popularity parameters: one set models the popularity of each row node across column communities, and the other models the popularity of each column node across row communities.

While \cite{jing2026two} considered a sub-Gaussian setting together with squared-loss-based estimation procedures, the present paper investigates the TNPM from a likelihood-based perspective under a Poisson formulation. Such a formulation is natural for modeling sparse networks with binary or multiple edges and leads to a variational inference framework that differs substantially from least-squares-based approaches. Since exact likelihood-based estimation is computationally intractable, we develop a variational EM procedure for inference under the TNPM. However, in contrast to the DC-LBM, the popularity parameters in the TNPM are intertwined across communities through the rank-one structure of the probability matrix, which substantially complicates both the identifiability analysis and the associated optimization problem. The proposed variational EM procedure combines computational efficiency with a likelihood-based inferential framework. Unlike existing approaches that rely on squared-loss formulations and multi-stage estimation strategies, our method directly optimizes a variational approximation to the likelihood under the TNPM.

We establish the label consistency of the variational estimator under the TNPM, allowing the expected graph density to approach zero. One of the most challenging parts of the proof is to establish the identifiability of community labels in the TNPM. In this model, each entry of the adjacency matrix is expressed as a product of link probabilities, which represent the popularity of nodes in one set with respect to the communities in the other. This structure naturally raises the question of identifiability: can an alternative assignment of community labels generate the same adjacency matrix? To address this, we impose conditions on the link probabilities, requiring that the link-probability vectors associated with any two distinct communities cannot be made identical through partitions and scalar rescaling. In other words, no community can be decomposed and reweighted to mimic the link-probability pattern of another. See Conditions \ref{cond:c1} and \ref{cond:c2} in Section~\ref{sec:theory}. We further refine Conditions \ref{cond:c1} and \ref{cond:c2} by allowing the discrepancies between the link-probability vectors of any two distinct communities to grow proportionally with the number of nodes. This refinement enforces a quantitative separation between communities, which in turn yields the desired well-separatedness of our objective functions.

In addition to the theoretical developments, our numerical experiments demonstrate that the proposed algorithm generally outperforms existing methods for fitting the TNPM in terms of clustering accuracy, including the Delete-One Method (DOM), which is more suitable for small-scale networks, and the Two-Stage Divided Cosine (TSDC) algorithm for large-scale networks. Moreover, it is worth noting that the proposed algorithm can also be adapted to fit the PABM for undirected networks with only minor adjustments.

Finally, we briefly review additional approaches for biclustering on bipartite and directed networks. \citet{rohe2016co} proposed a spectral co-clustering algorithm, DI-SIM, and established its theoretical properties under the stochastic co-block model (ScBM) as well as its extension—the degree-corrected ScBM. \cite{zhou2019analysis} introduced a data-driven regularization of the adjacency matrix and analyzed its theoretical properties under both the SBM and the directed SBM. \cite{Ji2020} established theoretical guarantees for the spectral method D-SCORE, originally proposed by \cite{Jin_DSCORE}, and for several of its variants under the directed DCSBM. In addition, several likelihood-based approaches have been proposed for biclustering on bipartite or directed networks. \cite{Zhang2022_directed} proposed a network embedding model that incorporates regularization on the embedding vectors to encourage clustering structures in the row and column embedding spaces, respectively. The profile-likelihood approach was applied to biclustering by \cite{wang2021fast}. 
\section{Model} \label{sec:model}
Let $G = (U, V, E)$ be a bipartite network, where $U = \{1, 2, \dots, m\}$ and $V = \{1, 2, \dots, n\}$ are disjoint sets of nodes (for example, $U$ can be the set of audience members and $V$ can be the set of movies), and $E \subseteq U \times V$ is the set of edges. In this paper, we allow multiple edges between two nodes. The network is represented by an $m \times n$ adjacency matrix $A$, where the entry $A_{ij} \in \mathbb{Z}_{\ge 0}$ denotes the number of edges between node $i \in U$ and node $j \in V$. 
Assume that the nodes in $U$ are partitioned into $K$ communities and those in $V$ into $L$ communities, with each node assigned to exactly one community. We refer to the communities associated with $U$ as row communities and those associated with $V$ as column communities. Let $z = \{z_1, z_2, \dots, z_m\}$ and $w = \{w_1, w_2, \dots, w_n\}$ denote the community labels of the nodes in $U$ and $V$, respectively.

We adopt the Two-Way Node Popularity Model (TNPM) proposed by \cite{jing2026two} to model the community structure of the bipartite network. We assume that $z_i \overset{\text{i.i.d.}}{\sim} \mathrm{Categorical}(\pi)$, where $\pi = (\pi_1, \dots, \pi_K)^\top$, with $\pi_k \geq 0$ for all $k$ and $\sum_{k=1}^K \pi_k = 1$, and that $w_j \overset{\text{i.i.d.}}{\sim} \mathrm{Categorical}(\rho)$, where $\rho = (\rho_1, \dots, \rho_L)^\top$, with $\rho_\ell \geq 0$ for all $\ell$ and $\sum_{\ell=1}^L \rho_\ell = 1$. These assignment labels can equivalently be represented by binary membership matrices  $1^{z} \in \mathbb{R}^{m \times K}$ and $1^{w} \in \mathbb{R}^{n \times L}$, where $1^{z}_{ik} = 1$ if $z_i = k$, and $1^{w}_{j\ell} = 1$ if $w_j = \ell$. Let $\theta \in \mathbb{R}^{m \times L}$ and $\lambda \in \mathbb{R}^{n \times K}$ be matrices of node-specific popularity parameters. The entry $\theta_{il}$ represents the popularity of node $i \in U$ with respect to the $l$-th community in $V$, and $\lambda_{jk}$ represents the popularity of node $j \in V$ with respect to the $k$-th community in $U$.
Given the community label assignments $z$ and $w$, the entries of $A$ are generated independently according to
$
A_{ij} \sim \mathrm{Poisson}(p_{ij})$, where $p_{ij} = \theta_{i w_j} \lambda_{j z_i}$.

The joint likelihood of the community labels $z$ and $w$ and the adjacency matrix $A$, given the model parameters, is
$$
P(z, w, A ; \pi, \rho, \theta, \lambda)
= \left( \prod_{i=1}^{m} \pi_{z_i} \right)
  \left( \prod_{j=1}^{n} \rho_{w_j} \right)
  \prod_{i=1}^{m} \prod_{j=1}^{n}
  e^{ -\theta_{i w_j} \lambda_{j z_i} }
  \frac{ \left( \theta_{i w_j} \lambda_{j z_i} \right)^{A_{ij}} }{ A_{ij}! }.
$$
As the primary objective of this paper is community detection, the community labels $z$ and $w$ are treated as latent variables throughout the analysis.
We can, at least in principle, integrate out the latent assignments by summing over all possible label configurations to obtain the marginal likelihood of the observed network $A$:
$$
P(A ; \pi, \rho, \theta, \lambda) = \sum_{z \in \Omega_z} \sum_{w \in \Omega_w} P(z, w, A ; \pi, \rho, \theta, \lambda),
$$ 
where $\Omega_{{z}}= \{1,...,K\}^m$ and $\Omega_{{w}}= \{1,...,L\}^n$.

The Popularity Adjusted Block Model (PABM) for undirected networks was first proposed by \cite{sengupta2018block}. Compared with the Degree-Corrected Stochastic Block Model (DCSBM), the PABM allows for more flexible connectivity patterns through community-specific popularity parameters, and was later extended by \cite{norooziSSC2021} to networks with an unknown number of communities. Motivated by the PABM, \cite{jing2026two} proposed the Two-Way Node Popularity Model (TNPM) for directed and bipartite networks. In this paper, we investigate the TNPM from a likelihood-based perspective and develop a variational Expectation–Maximization (EM) framework for statistical inference. Unlike existing approaches based on spectral methods or multi-stage estimation strategies, our method directly optimizes a variational approximation to the likelihood under the TNPM. The proposed approach is both computationally efficient and theoretically well grounded. However, unlike the Degree-Corrected Latent Block Model (DC-LBM), the TNPM induces a rank-one structure within each community block through community-specific popularity parameters, making both the identifiability analysis and parameter optimization substantially more challenging.

\section{Likelihood-based Variational Inference}\label{sec:algorithm}

A natural approach to estimating the latent variables $z$ and $w$ is to use the EM algorithm, which maximizes the marginal log-likelihood of the observed data by iteratively maximizing a lower bound on the likelihood. However, in the above setting, standard EM is impractical because the number of possible labels grows exponentially, making the posterior $P(z, w \mid A, \Phi)$ computationally intractable. To address this, we adopt a variational EM framework, initially introduced by \cite{govaert2008block} and rediscovered by \cite{wang2021efficient}, which introduces an approximation to the posterior over latent variables. 

\subsection{Variational EM algorithm}

Let $\Phi = (\pi, \rho, \theta, \lambda)$ denote the parameters and $q(z, w)$ be any distribution over the latent variables, then using Jensen's inequality, we can obtain the following bound:
\begin{align*}
    \text{log } P(A \mid \Phi) &= \text{log} \sum_{z \in \Omega_z} \sum_{w \in \Omega_w} q(z, w) \frac{P(z, w, A ; \Phi)}{q(z,w)} \\
    &\ge  \sum_{z \in \Omega_z} \sum_{w \in \Omega_w} q(z, w) \log \Bigg( \frac{P(z, w, A; \Phi)}{q(z, w)} \Bigg) =: J(q, \Phi),
\end{align*}
where $J(q, \Phi)$ is the evidence lower bound (ELBO). The above bound is tight only when $q(z, w) = P(z, w \mid A, \Phi)$, which corresponds to the E-step of the standard EM algorithm, where the exact posterior is computed. To make this computation tractable, \cite{govaert2008block}, \cite{wang2021efficient}, and, more recently, \cite{zhao2024variational} proposed imposing a factorization constraint on the variational distribution, specifically assuming $q(z, w) = q_1(z) q_2(w)$. This drastically simplifies the optimization by restricting the variational family to independent distributions over $z$ and $w$. 
This assumption is also intuitively justified by the observation that the posterior distribution of $(\bm{z}, \bm{w})$ will eventually concentrate on a single realization---the true cluster labels---as the sample sizes $m$ and $n$ become large. In other words, the posterior distribution of $(\bm{z}, \bm{w})$ is approximately a Dirac measure, which can be factorized by default. Under this factorization, the ELBO becomes: 
\begin{align*}
    J(q_1, q_2, \Phi) := \sum_{z} \sum_{w} q_1(z) q_2(w) \log \left( \frac{P(z, w, A | \Phi)}{q_1(z) q_2(w)} \right),
\end{align*}
and serves as the objective in the variational EM algorithm. It will be shown that maximizing this ELBO with respect to $q_1$ (resp. $q_2$), given the other distribution and the parameters, yields factorized solutions, thereby enabling efficient computation. In the rest of the section, we describe the details of the variational EM algorithm, which consists of alternating between the E-step and the M-step.

In the M step of the variational EM algorithm, we aim to maximize $J(q_1, q_2, \Phi)$ over all parameters $\Phi = (\pi, \rho, \theta, \lambda)$ given some fixed variational distributions $q_1$ and $q_2$ from the E-step. Up to a constant, $J(q_1, q_2, \Phi)$ with $q_1$ and $q_2$ fixed can be written as:
\begin{align}
J(q_1, q_2, \Phi) & = - \sum_{i=1}^{m} \sum_{j=1}^{n} \sum_{k=1}^{K} \sum_{l=1}^{L} P_{q_1}(z_i=k) P_{q_2}(w_j=l) \theta_{il} \lambda_{jk} \\
    & \quad + \sum_{i=1}^{m} \sum_{j=1}^{n} A_{ij} \sum_{k=1}^{K} \sum_{l=1}^{L} P_{q_1}(z_i=k) P_{q_2}(w_j=l) \text{log} (\theta_{il} \lambda_{jk}) \\ 
    & \quad + \sum_{i=1}^{m} \sum_{k=1}^{K} P_{q_1}(z_i=k) \text{log } \pi_{k} + \sum_{j=1}^{n} \sum_{l=1}^{L} P_{q_2}(w_j=l) \text{log } \rho_{l}, \label{objective}
\end{align}
where \begin{align*}
    P_{q_1}(z_i = k) = \sum_{z \in \Omega_z} q_1(z_1, \dots, z_{i-1}, k, z_{i+1}, \dots, z_m), i = 1, \dots, m, k = 1, \dots, K,
\end{align*}
and
\begin{align*}
    P_{q_2}(w_j = l) = \sum_{w \in \Omega_w} q_2(w_1, \dots, w_{j-1}, l, w_{j+1}, \dots, w_n), j = 1, \dots, n, l = 1, \dots, L.
\end{align*}
It is shown in Proposition \ref{prop:estep}  that maximizing $J(q_1, q_2, \Phi)$ with respect to $q_1$ (resp. $q_2$), given the other distribution and the parameters $\Phi$, yields factorized solutions; namely, $q_1(z)=\prod_{i=1}^m q_{1i}(z_i)$ (resp. $q_2(w)=\prod_{j=1}^n q_{2j}(w_j)$). Consequently, the marginal probabilities satisfy
$P_{q_1}(z_i = k) = q_{1i}(k)$ and
$P_{q_2}(w_j = l)=q_{2j}(l)$, which can therefore be computed directly. The details will be given in the E-steps. 

We now focus on the maximization with respect to the set of parameters $\Phi$, which consist of $\pi$, $\rho$, $\theta$, and $\lambda$. A major technical difference between the M-steps for the TNPM and the DC-LBM is that, in the TNPM, the popularity parameters $\theta$ and $\lambda$ cannot be separated from the label assignment distributions $q_1$ and $q_2$. This is because $\theta_{il}$ measures the popularity of node $i$ with respect to column community $l$ (and similarly $\lambda_{jk}$ for node $j$ with respect to row community $k$), and therefore inherently depends on the underlying community assignments. For the M-steps, we have following result: 
\begin{proposition}
\label{prop:mstep} For any fixed $q_1$ and $q_2$, $\hat\Phi=(\hat{\pi},\hat{{\rho}},\hat{{\theta}},\hat{{\lambda}}) $ satisfying the following equations is a global maximizer of $J(q_1,q_2,\Phi)$.
\begin{align}
    &\hat{\theta}_{il} = \frac{\sum_{j=1}^{n} A_{ij} P_{q_2}(w_j=l)}{\sum_{j=1}^{n} \sum_{k=1}^{K} P_{q_1}(z_i=k) P_{q_2}(w_j=l) \hat{\lambda}_{jk}}, \label{estimating_theta}  \\
    &\hat{\lambda}_{jk} = \frac{\sum_{i=1}^{m} A_{ij} P_{q_1}(z_i=k)}{\sum_{i=1}^{m} \sum_{l=1}^{L} P_{q_1}(z_i=k) P_{q_2}(w_j=l) \hat{\theta}_{il}}, \label{estimating_lambda} \\
    &\hat{\pi}_k = \frac{\sum_{i=1}^{m} P_{q_1}(z_i=k)}{m}, \quad \hat{\rho}_l = \frac{\sum_{j=1}^{n} P_{q_2}(w_j=l)}{n}. \nonumber 
\end{align}
\end{proposition}

Note that Equations \eqref{estimating_theta} and \eqref{estimating_lambda} define estimating equations rather than closed-form solutions for $\hat{\theta}_{il}$ and $\hat{\lambda}_{jk}$, in contrast to the corresponding updates in the DC-LBM. Typically, for given $q_1$ and $q_2$, one can perform an iterative algorithm that alternately updates $\theta$ and $\lambda$ based on \eqref{estimating_theta} and \eqref{estimating_lambda} by fixing the other until convergence. When the distributions $q_1$ and $q_2$ are degenerate---that is, there is exactly one nonzero value in $P_{q_1}(z_i = k)$ for $k = 1, \ldots, K$, and in $P_{q_2}(w_j = l)$ for $l = 1, \ldots, L$---it can be shown that the estimating equations \eqref{estimating_theta} and \eqref{estimating_lambda} admit closed-form solutions, as summarized in the following proposition:

\begin{proposition}
\label{prop:mstep2}
Suppose the distributions $q_1$ and $q_2$ are degenerate. Let $\mathcal{R}_k := \{ i : z_i = k \}$ and $\mathcal{C}_l := \{ j : w_j = l \}$ denote the $k$-th row and $l$-th column communities, respectively. Then $\hat{\theta}_{il}$ and $\hat{\lambda}_{jk}$, as defined below, constitute a global optimizer of $J$ given $q_1$ and $q_2$.
\begin{align*}
&\hat{\theta}_{il} = \frac{\sum_{j \in \mathcal{C}_l} A_{ij}}{\sqrt{\sum_{j \in \mathcal{C}_l} \sum_{i' \in \mathcal{R}_{z_i}} A_{i'j}}}, \,\, \hat{\lambda}_{jk} = \frac{\sum_{i \in \mathcal{R}_k} A_{ij}}{\sqrt{\sum_{i \in \mathcal{R}_k} \sum_{j' \in \mathcal{C}_{w_j}} A_{ij'}}}. 
\end{align*}
\end{proposition}
Note that, even as closed-form solutions, $\hat{\theta}$ and $\hat{\lambda}$ still depend on the label assignments, which differs from the case of the degree parameters in the DC-LBM. It is also interesting to note that the above solutions coincide with the estimates of the popularity parameters for undirected networks given the cluster assignments in \cite{sengupta2018block} (Equation (9)).

The E-step of the variational EM algorithm is similar to that described in \cite{zhao2024variational}. We include the details here for completeness.
In the E-step, our objective is to optimize $J(q_1, q_2, \Phi)$ with respect to the distributions $q_1(z)$ and $q_2(w)$ while holding the parameters $\Phi = (\pi, \rho, \theta, \lambda)$ fixed. A key result is that performing coordinate-wise maximization of $J(q_1, q_2)$ leads to factorized solutions, as presented in the proposition below, whose proof is identical to that of Proposition 2 in \cite{zhao2024variational} and is therefore omitted. 

\begin{proposition}
\label{prop:estep}
Define
\begin{align*}
    g_1(z_i) & = - \sum_{j=1}^{n} \lambda_{jz_i} \left( \sum_{l = 1}^{L} P_{q_2}(w_j = l) \theta_{il} \right) + \sum_{j=1}^{n} A_{ij} \left( \sum_{l = 1}^{L} P_{q_2}(w_j = l) \log (\theta_{il}) \right) \\  &  \qquad   \qquad  \qquad  \qquad \qquad  + \sum_{j=1}^{n} A_{ij} \log (\lambda_{jz_i}) + \log \pi_{z_i}, \hspace{1em} i = 1, \dots, m.
\end{align*}
\begin{align*}
    g_2(w_j) & = - \sum_{i=1}^{m} \theta_{iw_j} \left( \sum_{k = 1}^{K} P_{q_1}(z_i = k) \lambda_{jk} \right) + \sum_{i=1}^{m} A_{ij} \left( \sum_{k = 1}^{K} P_{q_1}(z_i = k) \log (\lambda_{jk}) \right) \\ &  \qquad   \qquad  \qquad  \qquad \qquad +   \sum_{i=1}^{m} A_{ij} \log (\theta_{iw_j}) + \log \rho_{w_j}, \hspace{1em} j = 1, \dots, n.
\end{align*}
Then, given $\Phi$ and $q_2$, 
\begin{equation*}
\operatorname*{arg max}_{q_1} J(q_1, q_2, \Phi) = \prod_{i=1}^{m} \frac{e^{g_1(z_i)}}{\sum_{k=1}^{K} e^{g_1(k)}},
\end{equation*}
and, given $\Phi$ and $q_1$,
\begin{equation*}
 \operatorname*{arg max}_{q_2} J(q_1, q_2, \Phi) = \prod_{j=1}^{n} \frac{e^{g_2(w_j)}}{\sum_{l=1}^{L} e^{g_2(l)}}.
\end{equation*}
\end{proposition}

\subsection{Initial Values}\label{initial_values}
We initialize the distributions in the variational EM $q_1(z)$ and $q_2(w)$ using hard assignments, that is, $q_1(z) = \prod_{i=1}^m \mathds{1}(z_i = \hat{z}^{\mathrm{init}}_i)$ and $q_2(w) = \prod_{j=1}^n \mathds{1}(w_j = \hat{w}^{\mathrm{init}}_j)$, where $\hat{z}^{\mathrm{init}}$ and $\hat{w}^{\mathrm{init}}$ are labels obtained using certain eigen-based methods. Following the approach of \cite{jing2026two}, the cluster labels $z^{\mathrm{init}}$ and $w^{\mathrm{init}}$ are initialized using singular value decomposition (SVD) of the adjacency matrix $A$. We first compute the SVD $A = U\Sigma V^T$, where $U$ and $V$ contain the left and right singular vectors, respectively. For the initial row clusters, we take the first $K$ columns of $U$ to form $U_K \in \mathbb{R}^{m \times K}$, then apply $K$-means to its rows. Similarly, for the column clusters, we use the first $L$ columns of $V$ to form $V_L \in \mathbb{R}^{n \times L}$ and cluster its rows with $K$-means. \cite{jing2026two} reported that this initialization method produces better results than alternative approaches, and we observed a similar pattern in our simulations. To further mitigate the influence of initialization, we use 10 random label initializations in addition to the SVD-based approach and select the final solution that yields the largest value of the objective function $J(q_1, q_2, \Phi)$. We begin the algorithm with the M-step. Note that $\theta$ and $\lambda$ have closed-form expressions according to Proposition \ref{prop:mstep2} in the first M-step, as the initial distributions $q_1(z)$ and $q_2(w)$ are degenerate.

\begin{algorithm}[H]
\caption{Variational EM}
\label{alg:variational-em}
\SetAlgoLined
\KwIn{$A$, $K$, $L$, number of random initializations $N=10$, tolerance $\varepsilon$}
\For{$s \gets 1$ \KwTo $N+1$}{
    \eIf{$s = 1$}{
        Initialize $\hat{z}^{(s)}$ and $\hat{w}^{(s)}$ using SVD, according to Section \ref{initial_values}\;
    }{
        Initialize $\hat{z}^{(s)}$ and $\hat{w}^{(s)}$ by uniform random assignment to $K$ and $L$ clusters, respectively\;
    }
    Initialize $q_1^{(s)}(z) \propto \prod_{i=1}^{m} 1(z_i = \hat{z}_i^{(s)})$ and $q_2^{(s)}(w) \propto \prod_{j=1}^{n} 1(w_j = \hat{w}_j^{(s)})$\;
    Initialize $\hat{\theta}_i^{(s)}, \hat{\lambda}_j^{(s)}$ using Proposition \ref{prop:mstep2} and initialize $\hat{\pi}_k^{(s)}, \hat{\rho}_l^{(s)}$ uniformly\;
    
    \Repeat{convergence}{
        \textbf{E-step:}\;
        $q_1^{(s)}(z) \gets
        \prod_{i=1}^{m}
        \frac{\exp\{g_1^{(s)}(z_i)\}}
        {\sum_{k=1}^{K}\exp\{g_1^{(s)}(k)\}}$\;
        $q_2^{(s)}(w) \gets
        \prod_{j=1}^{n}
        \frac{\exp\{g_2^{(s)}(w_j)\}}
        {\sum_{l=1}^{L}\exp\{g_2^{(s)}(l)\}}$\;
        \textit{Here, $g_1^{(s)}$ and $g_2^{(s)}$ are defined in Proposition~\ref{prop:estep}.}\;
        \textbf{M-step:}\;
        $\hat{\pi}_k^{(s)} \gets \frac{\sum_{i=1}^{m} P_{q_1^{(s)}}(z_i=k)}{m}$, and $\hat{\rho}_l^{(s)} \gets \frac{\sum_{j=1}^{n} P_{q_2^{(s)}}(w_j=l)}{n}$\;
        \Repeat{convergence}{
            $\hat{\theta}_{il}^{(s)} \gets \frac{\sum_{j=1}^{n} A_{ij} P_{q_2^{(s)}}(w_j=l)}{\sum_{j=1}^{n} \sum_{k=1}^{K} P_{q_1^{(s)}}(z_i=k) P_{q_2^{(s)}}(w_j=l) \hat{\lambda}_{jk}^{(s)}}$\;
            $\hat{\lambda}_{jk}^{(s)} \gets \frac{\sum_{i=1}^{m} A_{ij} P_{q_1^{(s)}}(z_i=k)}{\sum_{i=1}^{m} \sum_{l=1}^{L} P_{q_1^{(s)}}(z_i=k) P_{q_2^{(s)}}(w_j=l) \hat{\theta}_{il}^{(s)}}$\;
        }
    }
    Compute $J(q_1^{(s)}, q_2^{(s)}, \Phi^{(s)})$ using equation \eqref{objective} \;
}
\KwOut{$q_1^{(s^*)}$, $q_2^{(s^*)}$ $\text{ where } s^* \gets \operatorname*{arg\,max}_s \hat{J}(q_1^{(s)}, q_2^{(s)}, \Phi^{(s)})$ \;}
\end{algorithm}

\section{Identifiability and Consistency}\label{sec:theory}

In our theoretical analysis, we consider the solutions $q_1(z)$ and $q_2(w)$ in a factored form based on Proposition \ref{prop:estep}. Accordingly, $q_1(z)$ can be represented by an $m \times K$ matrix $[q_{ik}^z]$, where each entry $q_{ik}^z \in [0, 1]$ denotes the probability that row $i$ is assigned to cluster $k$ under $q_1(z)$. Similarly, $q_2(w)$ can be represented by an $n \times L$ matrix $[q_{jl}^w]$, where each entry $q_{jl}^w \in [0, 1]$ denotes the probability that column $j$ is assigned to cluster $l$ under $q_2(w)$. The similarity between the estimated cluster allocation and the true cluster labels is measured using the following soft confusion matrices: 

\begin{definition}[Soft confusion matrix] 
\textit{For any row label assignment matrices $q^z$ and $\tilde{q}^z$, let} $
\mathbb{R}_{kk'}(q^z, \tilde{q}^z) = \frac{1}{m} \sum_{i=1}^m q_{ik}^z \tilde{q}_{ik'}^z.$ \textit{In particular, the confusion matrix for the true row label $z^*$ and $q^z$ is}
$
\mathbb{R}_{kk'}(1^{z^*}, q^z) = \frac{1}{m} \sum_{i=1}^m 1^{z^*}_{ik} q_{ik'}^z,$ \textit{where $1^{z}$ is a 0–1 matrix for any hard labeling on rows, with $1^{z}_{ik} = \mathds{1}(z_i = k)$.} 

\textit{Similarly, for any column label assignments, we define}
$
\mathbb{R}_{ll'}(q^w, \tilde{q}^w) = \frac{1}{n} \sum_{j=1}^n q_{jl}^w \tilde{q}_{jl'}^w$. \textit{Consequently,} $\mathbb{R}_{ll'}(1^{w^*}, q^w) = \frac{1}{n} \sum_{j=1}^n 1^{w^*}_{jl} q_{jl'}^w,
$ \textit{where $1^{w}$ is a 0-1 matrix for any hard labeling on columns, with $1^{w}_{jl} = \mathds{1}(w_j = l)$.}
\end{definition}
It is useful to present an alternative expression for the confusion matrix associated with the true row label assignment $z^{*}$ and another hard labeling $z$. For $k=1, \dots, K$, define  $\mathcal{A}_{k}^{*} = \{ 1 \leq i \leq m: z^{*}_i =k \}$ and $\mathcal{A}_{k} = \{ 1\leq i \leq m : z_i =k \}$. Then, it follows directly that
\begin{equation}
    \mathbb{R}_{k k'}( 1^{z^*}, 1^{z} ) = \frac{1}{m} \left| \mathcal{A}_{k}^{*} \cap \mathcal{A}_{k'} \right|,  \label{eqid:anotherformfoR}
\end{equation}
where $|\mathcal{A}|$ denotes the cardinality of the set $\mathcal{A}$. Similarly, for the true column label assignment $w^{*}$ and another labeling $w$, we have
\begin{equation*}
    \mathbb{R}_{l l'}( 1^{w^*}, 1^{w} ) = \frac{1}{n} \left| \mathcal{B}_{l}^{*} \cap \mathcal{B}_{l'} \right|,
\end{equation*}
where $\mathcal{B}_{l}^{*} = \{ 1 \leq j \leq n: w^{*}_j =l \}$ and $\mathcal{B}_{l} = \{ 1 \leq j \leq n : w_j =l \}$. The soft confusion matrices defined above generalize the confusion matrix used for comparing two hard label assignments to the case of comparing two probability matrices. 

Let $S_K$ (resp. $S_L$) denote the set of permutations on ${1, \dots, K}$ (resp. ${1, \dots, L}$). Taking permutations into account, the misclustering rates for row and column clusters are defined as follows:
\begin{equation*}
\begin{aligned}
M_{\mathrm{row}}(q^z) & = \min_{s \in S_K} \left( 1 - \sum_{k'=1}^K \mathbb{R}_{s(k'),k'}(1^{z^*}, q^z) \right), \\
M_{\mathrm{col}}(q^w) &= \min_{t \in S_L} \left( 1 - \sum_{l'=1}^L \mathbb{R}_{t(l'),l'}(1^{w^*}, q^w) \right).
\end{aligned}
\end{equation*}

For any hard row label assignment $z$ and column hard label assignment $w$, it follows directly that
\begin{equation}
\begin{aligned}
M_{\mathrm{row}}(1^z) & = \min_{s \in S_K} \left( 1 - \frac{1}{m} \sum_{k'=1}^K \left| \mathcal{A}_{s(k')}^{*} \cap \mathcal{A}_{k'} \right|  \right),  \\ M_{\mathrm{col}}(1^w) & = \min_{t \in S_L} \left( 1 - \frac{1}{n} \sum_{l'=1}^L \left| \mathcal{B}_{t(l')}^{*} \cap \mathcal{B}_{l'} \right| \right). \label{eqid:anotherformofMrowcol}
\end{aligned}
\end{equation}

We assume that all parameters, including $q_z$ and $q_w$, are contained in compact domains: $\mathcal{C}_z = \{ q^z \in \mathbb{R}^{m \times K} : q_{ik}^z \in [0,1],\, \sum_k q_{ik}^z = 1, \forall i \}$, $\mathcal{C}_w = \{ q^w \in \mathbb{R}^{n \times L} : q_{jl}^w \in [0,1],\, \sum_l q_{jl}^w = 1, \forall j \}$, $\mathcal{C}_\pi = \{ \pi \in \mathbb{R}^K : \pi_k \in [\pi_{\min}, \pi_{\max}],\, \sum_k \pi_k = 1 \}$, $\mathcal{C}_\rho = \{ \rho \in \mathbb{R}^L : \rho_l \in [\rho_{\min}, \rho_{\max}],\, \sum_l \rho_l = 1 \}$, $\mathcal{C}_\theta = \{ \theta \in \mathbb{R}^{m \times L} : \theta_{il} \in [\eta_{\min}, \eta_{\max}], \forall i, l \}$, $\mathcal{C}_\lambda = \{ \lambda \in \mathbb{R}^{n \times K} : \lambda_{jk} \in [\eta_{\min}, \eta_{\max}], \forall j, k \}$. Among these interval limits, $\pi_{\min}, \pi_{\max}, \rho_{\min}$, and $\rho_{\max}$ are assumed to be constant and bounded away from 0. By contrast, $\eta_{\min}$ and $\eta_{\max}$ may go to zero because they control the overall graph density. We assume $\eta_{\min} \asymp \eta_{\max}$---that is, these quantities are of the same order. Let $\tilde{\eta}$ denote a quantity of this common order.

The main ingredient of the proof is an analysis of the population version of the objective function, which is defined as the expectation of the objective with respect to $(q^z, q^w, \theta, \lambda)$, with lower-order terms omitted:
\begin{align*}
 \bar{J}(q^z, q^w, \theta, \lambda) &= -\sum_{i=1}^m \sum_{j=1}^n \left( \sum_{k=1}^K \sum_{l=1}^L q_{ik}^z q_{jl}^w \theta_{i l} \lambda_{j k} \right) \\ 
 &\quad + \sum_{i=1}^m \sum_{j=1}^n E[A_{ij} \mid z^*, w^*] \left( \sum_{k=1}^K \sum_{l=1}^L q_{ik}^z q_{jl}^w \log (\theta_{i l} \lambda_{j k}) \right). 
\end{align*}


It is straightforward to show that the true parameter and label assignment constitute the maximizer of the population version of the objective function, as expected.  However, such a result does not guarantee the uniqueness of the maximizer. The goal of the next theorem is to address this issue by showing that $1^{z^*}$ and $1^{w^*}$ uniquely maximize the function, up to label permutations, under certain conditions.

Before we present the conditions, we pause to introduce the necessary notation. We define
\begin{align*}
 & \mathcal{I}_{z} := \{ 1^{z} \in \mathbb{R}^{m \times K} : \text{for $\forall i$, one of $1^{z}_{i1}, \dots, 1^{z}_{iK}$ is $1$ and others are $0$}  \}, \\
 & \mathcal{I}_{w} := \{ 1^{w} \in \mathbb{R}^{n \times L} : \text{for $\forall j$, one of $1^{w}_{j1}, \dots, 1^{w}_{jL}$ is $1$ and others are $0$}  \}.
\end{align*}
For any $q^z \in \mathcal{C}_{z}$, we define
\begin{equation*}
\mathcal{I}_{z}( q^z ) := \{  1^{z} \in \mathcal{I}_{z}: \text{$1^{z}_{ik} =1$ only for $i$ and $k$ such that $q_{ik}^z>0$} \}.
\end{equation*}
In other words, $ \mathcal{I}_{z}( q^z ) $ is a subset of $ \mathcal{I}_{z}$ whose elements must satisfy the condition that if a particular entry in the matrix is $1$, then the corresponding entry in $q^z$ must be positive. Similarly, for any $q^w \in \mathcal{C}_{w}$, let
\begin{equation*}
\mathcal{I}_{w}( q^w ) := \{  1^{w} \in \mathcal{I}_{w}: \text{$1^{w}_{jl} =1$ only for $j$ and $l$ such that $q^{w}_{jl} >0$} \}.
\end{equation*}

We next introduce a new metric adapted from \cite{zhao2024variational}. Let $\mu = \eta_{\max}^2/\eta_{\min}$. For any $a,b\in\mathbb{R}$, define $|a - b|_{\mathrm{new}} := \min \{ |a-b|, 6\mu \}$. It is straightforward to verify that this defines a valid metric on $\mathbb{R}$. Furthermore, for vectors $a$ and $b$ in $ \mathbb{R}^{d} $, define $
\| a- b \|_{\mathrm{new}} = \sqrt{ \sum_{i=1}^{d} |a_i - b_i|_{\mathrm{new}}^{2} }$. 

Finally, we introduce standard notation for submatrices in linear algebra. Suppose that $P= (p_{ij})_{ 1 \leq i \leq m, 1 \leq j \leq n } \in \mathbb{R}^{m \times n}$ is a matrix. Furthermore, let $\mathcal{A}$ be a subset of $\{ 1, \dots, m\}$ and $\mathcal{B}$ a subset of $\{1, \dots, n\}$. We denote by $P_{\mathcal{A}, \mathcal{B}}$ the submatrix of $P$ formed by selecting the rows indexed by $\mathcal{A}$ and the columns indexed by $\mathcal{B}$; explicitly, $P_{\mathcal{A}, \mathcal{B}} = (p_{ij})_{i \in \mathcal{A}, j \in \mathcal{B}}$. For simplicity, we use $P_{\mathcal{A}, j}$ to denote $P_{\mathcal{A}, \{j\}}$, and $P_{i, \mathcal{B} }$ to denote $P_{ \{i\}, \mathcal{B} }$.

With the preparations above and recalling the definitions of $ \mathcal{A}^{*}_{k} $ and $ \mathcal{B}^{*}_{l} $ at the beginning of Section \ref{sec:theory}, we are ready to present our conditions.
\begin{enumerate}[label=(C\arabic*)]
\item \label{cond:c1} Suppose $|\mathcal{A}_{k}^{*}| >0$, for any $k =1, \dots, K$, and $|\mathcal{B}_{l}^{*}| >0$ for any $l = 1, \dots, L$.
\item \label{cond:c2} Suppose there exists some $k \in \{1, \dots, K\}$ such that for all $l_1, l_2 \in \{1, \dots, L\}$ satisfying $l_1 \neq l_2$, we have
\begin{equation}
\inf_{\substack{ \mathcal{A}_{k1}, \dots, \mathcal{A}_{kK} \, \text{form} \\ \text{ a partition of $\mathcal{A}_{k}^{*}$}  }} \  \inf_{ c_1, \dots, c_K \in \left[ \frac{ \eta_{\mathrm{min}} }{\eta_{\mathrm{max}}} ,  \frac{ \eta_{\mathrm{max}} }{\eta_{\mathrm{min}}}\right] } \  \sum_{k'=1}^{K} \left\| \theta^{*}_{ \mathcal{A}_{k k'}, \, l_1 } - c_{k'} \theta^{*}_{ \mathcal{A}_{k k'}, \, l_2 } \right\|^{2}_{\mathrm{new}} >0. \label{eqid:firstpartC2}
\end{equation}
Similarly, suppose there exists some $l \in \{1, \dots, L\}$ such that for all $k_1, k_2 \in \{1, \dots, K\}$ satisfying $k_1 \neq k_2$, we have
\begin{equation}
\inf_{\substack{ \mathcal{B}_{l1}, \dots, \mathcal{B}_{lL} \, \text{form} \\ \text{ a partition of $\mathcal{B}_{l}^{*}$}  }} \  \inf_{ c_1, \dots, c_L \in \left[ \frac{ \eta_{\mathrm{min}} }{\eta_{\mathrm{max}}} ,  \frac{ \eta_{\mathrm{max}} }{\eta_{\mathrm{min}}}\right] } \  \sum_{l'=1}^{L} \left\| \lambda^{*}_{ \mathcal{B}_{l l'}, \, k_1 } - c_{l'} \lambda^{*}_{ \mathcal{B}_{l l'}, \, k_2 } \right\|^{2}_{\mathrm{new}} >0. \label{eqid:secondpartC2}
\end{equation}
Note that some subsets in the partition may be empty. For example, it is possible that $\mathcal{A}_{k1} = \varnothing$. In such cases, $ \| \theta^{*}_{ \mathcal{A}_{k 1}, \, l_1 } - c_{1} \theta^{*}_{ \mathcal{A}_{k 1}, \, l_2 } \|_{\mathrm{new}} $ is conventionally defined to be zero. The same convention applies to terms involving an empty subset $\mathcal{B}_{ll'}$.
\end{enumerate} 


\begin{remark}
Condition~\ref{cond:c2} admits a geometric interpretation. Fix a true row community $\mathcal{A}_{k}^{*}$ and two distinct column indices $l_1,l_2$. Since there are finitely many partitions of $\mathcal{A}_{k}^{*}$, and the scaling constants range over a compact set, the infimum in \eqref{eqid:firstpartC2} is attained. It equals zero if and only if there exist a partition $\mathcal{A}_{k1},\dots,\mathcal{A}_{kK}$ and constants $c_1,\ldots,c_K\in[\eta_{\min}/\eta_{\max},\eta_{\max}/\eta_{\min}]$ such that
\[
\theta^*_{\mathcal{A}_{kk'},l_1}
=
c_{k'}\theta^*_{\mathcal{A}_{kk'},l_2},
\qquad k'=1,\ldots,K.
\]
Because all entries of $\theta^*$ are strictly positive, this implies that the row vectors of $\theta^*_{\mathcal{A}_k^*,\{l_1,l_2\}}$ lie on at most $K$ distinct rays emanating from the origin. Thus, under the nonemptiness requirement in (C1), the first part of Condition~\ref{cond:c2} is satisfied if there exists a true row community $\mathcal{A}_k^*$ such that, for every pair $l_1\ne l_2$, the corresponding row vectors occupy at least $K+1$ distinct rays. Similarly, the second part of Condition~\ref{cond:c2} holds if there exists a true column community $\mathcal{B}_l^*$ such that, for every pair $k_1\ne k_2$, the row vectors of $\lambda^*_{\mathcal{B}_l^*,\{k_1,k_2\}}$ occupy at least $L+1$ distinct rays emanating from the origin.
\end{remark}

With these conditions in hand, we proceed to state the identifiability result for the maximizer $1^{z^*}, 1^{w^*}$.
\begin{theorem}  \label{thm:identification}
Suppose Conditions \ref{cond:c1} and \ref{cond:c2} hold. Then, for any $\theta \in \mathcal{C}_{\theta}$ and $\lambda \in \mathcal{C}_{\lambda}$,
\begin{equation*}
\bar{J} \left( 1^{z^*}, 1^{w^*}, \theta^*, \lambda^* \right) - \bar{J} \left( q^z, q^w, \theta, \lambda \right) =0
\end{equation*}
only if $ M_{\mathrm{row}} (q^z) =0 $ and $M_{\mathrm{col}} (q^{w}) =0$.
\end{theorem}


To establish label consistency, we require a property stronger than the uniqueness of the maximizer $(1^{z^*},1^{w^*})$, namely, a suitable well-separatedness property. More precisely, we aim to show that the gap $\bar{J}\left(1^{z^*},1^{w^*},\theta^*,\lambda^*\right) -\bar{J}\left(q^z,q^w,\theta,\lambda\right)$
can be bounded below by positive constant multiples of $mnM_{\mathrm{row}}(q^z)$ and $mnM_{\mathrm{col}}(q^w)$, respectively. To this end, we introduce a new condition on $\theta^*$ and $\lambda^*$ that strengthens Condition~\ref{cond:c2}. Under this stronger condition, we establish the desired well-separatedness property of $\bar{J}$.

\begin{enumerate}[label=(C3)]
\item \label{cond:c3} There exists $\tau$ of the same order as $\eta_{\max}^2$ such that, for all $k \in \{1, \dots, K\}$ and all $l_1, l_2 \in \{1, \dots, L\}$ with $l_1 \neq l_2$,
\begin{equation*}
\inf_{\substack{ \mathcal{A}_{k1}, \dots, \mathcal{A}_{kK} \, \text{form} \\ \text{ a partition of $\mathcal{A}_{k}^{*}$}  }} \  \inf_{ c_1, \dots, c_K \in \left[ \frac{ \eta_{\mathrm{min}}^{2} }{\eta_{\mathrm{max}}^{2}} ,  \frac{ \eta_{\mathrm{max}}^{2} }{ \eta_{\mathrm{min}}^{2}}\right] } \  \sum_{k'=1}^{K} \left\| \theta^{*}_{ \mathcal{A}_{k k'}, \, l_1 } - c_{k'} \theta^{*}_{ \mathcal{A}_{k k'}, \, l_2 } \right\|^{2}_{\mathrm{new}} \geq  \tau \left| \mathcal{A}_{k}^{*} \right| ,
\end{equation*}
and for all $l \in \{1, \dots, L\}$ and all $k_1, k_2 \in \{1, \dots, K\}$ with $k_1 \neq k_2$, 
\begin{equation}
\inf_{\substack{ \mathcal{B}_{l1}, \dots, \mathcal{B}_{lL} \, \text{form} \\ \text{ a partition of $\mathcal{B}_{l}^{*}$}  }} \  \inf_{ c_1, \dots, c_L \in \left[ \frac{ \eta_{\mathrm{min}}^{2} }{\eta_{\mathrm{max}}^{2}} ,  \frac{ \eta_{\mathrm{max}}^{2} }{\eta_{\mathrm{min}}^{2}}\right] } \  \sum_{l'=1}^{L} \left\| \lambda^{*}_{ \mathcal{B}_{l l'}, \, k_1 } - c_{l'} \lambda^{*}_{ \mathcal{B}_{l l'}, \, k_2 } \right\|^{2}_{\mathrm{new}} \geq \tau \left| \mathcal{B}_{l}^{*} \right| .    \label{eqws:secondpartC3}
\end{equation}
\end{enumerate}

\begin{remark}
Under Condition (C1), Condition (C3) implies Condition (C2). Furthermore, the requirement that $\tau$ be of the same order as $\eta_{\max}^2$ reflects the natural scale of the squared discrepancies in (C3). Indeed, the entries of $\theta^*$ are of order $\eta_{\max}$, and the admissible scaling coefficients $c_{k'}$ are uniformly bounded because $\eta_{\min} \asymp \eta_{\max}$. Thus, $\eta_{\max}^2$ is the natural scale for each squared coordinatewise discrepancy in $\big\| \theta^{*}_{\mathcal{A}_{kk'},l_1} - c_{k'}\theta^{*}_{\mathcal{A}_{kk'},l_2} \big\|_{\mathrm{new}}^2$. Summing over the partition yields the natural scale $\eta_{\max}^2 |\mathcal{A}_k^*|$, which motivates the choice $\tau \asymp \eta_{\max}^2$ in the lower bound. The same reasoning applies to $\lambda^*$ and the corresponding discrepancies over $\mathcal{B}_l^*$.
\end{remark}





Recall that the constant $\mu$, defined as $\mu =\eta_{\max}^2/\eta_{\min}$, appears in the definition of the new norm. The following theorem establishes the well-separatedness of the population version of the objective function $\bar{J}$.
\begin{theorem} \label{thm:wellseparatedness}
Suppose Conditions~\ref{cond:c1} and~\ref{cond:c3} hold.
For all $q^z \in \mathcal{C}_z, q^w \in \mathcal{C}_w,$ $\theta \in \mathcal{C}_\theta,$ and $\lambda \in \mathcal{C}_\lambda$, we have
\begin{align*}
    \bar{J}(1^{z^*}, 1^{w^*}, \theta^*, \lambda^*) - \bar{J}(q^z, q^w, \theta, \lambda) &\geq   C_1 m n M_{\text{row}}(q^z), \\
    \bar{J}(1^{z^*}, 1^{w^*}, \theta^*, \lambda^*) - \bar{J}(q^z, q^w, \theta, \lambda) &\ge C_2 m n M_{\text{col}}(q^w),
\end{align*}
where $C_1 = \frac{\tau}{12 \mu L} \frac{ \eta_{\mathrm{min}}^{5} }{ \eta_{\mathrm{min}}^{4} + \eta_{\mathrm{max}}^{4} }$ and $C_2 = \frac{\tau}{12 \mu K} \frac{ \eta_{\mathrm{min}}^{5} }{ \eta_{\mathrm{min}}^{4} + \eta_{\mathrm{max}}^{4} }$.
\end{theorem}


By combining Theorem \ref{thm:wellseparatedness} with a concentration result connecting the population and sample versions of the objective function, we obtain the main result of this section---label consistency.
Let
\[
(\hat{q}^{z}, \hat{q}^{w}, \hat{\Phi}) = \mathop{\arg\max}_{q^{z} \in \mathcal{C}_{z},\, q^{w} \in \mathcal{C}_{w},\, \theta \in \mathcal{C}_{\theta},\, \lambda \in \mathcal{C}_{\lambda},\, \pi \in \mathcal{C}_{\pi}, \rho \in \mathcal{C}_{\rho}} J(q^{z}, q^{w}, \Phi),
\]
where $\hat{\Phi} = (\hat{\theta}, \hat{\lambda}, \hat{\pi}, \hat{\rho})$. Moreover, recall that $\tilde{\eta}$ denotes the common order of $\eta_{\min}$ and $\eta_{\max}$.
\begin{theorem}
\label{thm:labelconsistency}
Under Conditions~\ref{cond:c1} and~\ref{cond:c3},
if $mn \tilde{\eta}^2 / ((m+n)(\log \tilde{\eta})^2) \rightarrow \infty$ as $m, n \rightarrow \infty$, then,
$M_\textnormal{row}(\hat{q}^{{z}})=o_p(1)$, and $M_\textnormal{col}(\hat{q}^{{w}})=o_p(1)$.
\end{theorem}
It is worth noting that $\tilde{\eta}^2$ controls the overall graph density. Up to a logarithmic term, the above condition is consistent with the typical assumptions on graph density in much of the community detection literature; see, for example, \cite{Bickel&Chen2009}. 
Furthermore, the identifiability and consistency results concern only the community labels and do not establish corresponding properties for the popularity parameters $\{\theta_i\}$ and $\{\lambda_j\}$.

\section{Simulation Studies}\label{sec:simulation1}
In this section, we conduct simulation studies to evaluate the performance of the proposed variational EM algorithm for community detection. In the first subsection, we evaluate the performance of the proposed method in bipartite networks, which represents the algorithm's primary application. As a by-product, the proposed method also shows effective performance on undirected networks, which will be discussed in the second subsection. 

We compare the proposed method against three competing approaches: Singular Value Decomposition followed by K-means (SVDK), which applies K-means clustering to the leading singular vectors of the observed adjacency matrix; TSDC, a two-stage divided cosine algorithm that iteratively updates node labels based on cosine similarity to community centers; and DOM, a delete-one method that minimizes a rank-one approximation error. The latter two methods were proposed in \cite{jing2026two}. All three iterative methods---VEM, TSDC, and DOM---are initialized from the same SVDK output. The proposed VEM additionally runs multiple random initializations as described in Section~\ref{initial_values}, selecting the solution with the highest variational lower bound.  The quality of estimated community labels is measured by the Adjusted Rand Index (ARI) \citep{vinh2010information}, which ranges between 0 and 1 with values near 0 indicating agreement no better than random chance and values near 1 indicating near-perfect recovery. 
In the bipartite setting, ARI is computed separately for row and column communities. 

\subsection{Bipartite Networks}
To evaluate our method on bipartite networks, we adopt the following framework: Data are generated as $A_{ij} \sim \text{Poisson}(P_{ij})$ where $P_{ij} = r \cdot \theta_{iw_j} \cdot \lambda_{jz_i}$, with $z_i \in \{1, \dots, K\}$ and $w_j \in \{1, \dots, L\}$ denoting the true community memberships of the row node $i$ and the column node $j$, respectively. We set $m = 800, n= 1000, K=3, L=4$, with communities balanced and memberships assigned uniformly at random. The popularity parameters $\theta_{il}$ and $\lambda_{jk}$ are drawn independently across nodes and communities. For each combination of distributional condition and sparsity level, we perform 100 replicates, resampling all parameters $\theta, \lambda, z$ and $w$ each time. Here, $\lambda$ and $\theta$ are randomly generated in each replicate to enhance the representativeness and generality of the simulation study. Nevertheless, these parameters are not treated as random quantities in the methodological or theoretical framework of this paper.


We consider three distributional setups for the popularity parameters, designed to induce increasing levels of degree heterogeneity in the adjacency matrix $A$. In the first setup, entries are drawn from a Uniform$(0, 1)$ distribution, producing a bounded and relatively homogeneous degree distribution as a baseline. In the second setup, entries follow a Gamma distribution with $\alpha = 0.5$, which is highly right-skewed and induces moderate degree heterogeneity. In the third setup, entries follow a Pareto distribution with shape parameter $\alpha = 2$ and scale parameter $x_{\min} = 0.25$, producing a heavy-tailed degree distribution with infinite variance and a pronounced hub structure. The scale parameter is analytically calculated to enforce the mean constraint below.  The three setups thus form a progression: bounded homogeneity, right-skewness with finite moments, and heavy-tailed behavior with infinite variance, which is representative of real-world sparse networks \citep{Albert2002}. Across all three setups, the parameters are constrained to satisfy $\mathbb{E}[\theta_{il}] = \mathbb{E}[\lambda_{jk}] = 0.5$, holding the expected network density fixed at $\mathbb{E}[P_{ij}] = r/4$ so that performance differences across conditions reflect degree heterogeneity rather than density. These simulation designs are not intended to enforce the asymptotic regularity conditions used in our theoretical analysis. In particular, conditions such as the comparability of the minimum and maximum popularity parameters are sufficient conditions imposed on sequences of parameters as the network size grows, whereas the simulations are designed to evaluate finite-sample performance under a broader range of heterogeneity. We deliberately include increasingly skewed and heavy-tailed settings as stress tests to examine how the proposed method performs under these settings. 


The sparsity parameter $r \in \{0.1, 0.2, 0.3, 0.4, 0.5\}$ directly controls the overall signal strength. Since $\mathbb{E}[P_{ij}] = r/4$, the marginal probability of a zero entry is approximately $\exp(-r/4)$, ranging from 97.5\% at $r=0.1$ to 88.2\% at $r=0.5$. The most challenging setting, $r=0.1$, therefore yields a matrix that is over 97\% zeros. Robustness to alternative choices of $m, n, K$, and $L$ is examined in Section~\ref{supp:robustness} of the supplementary material.

Figure \ref{fig:clustering_accuracy} reports the mean ARI ($\pm$ 2 SE) for row and column community recovery across the three distributional conditions and five values of the sparsity parameter $r$. VEM achieves the highest or near-highest ARI across the vast majority of conditions, attaining near-perfect recovery under the uniform and gamma conditions for $r \ge 0.2$ and improving steadily with $r$ under the power law condition. This reflects VEM's model-based objective, which explicitly accounts for node-specific popularity; TSDC's reliance on cosine similarity can provide insufficient signal at low $r$. DOM lags behind both methods throughout. SVDK recovers row communities reasonably well under the uniform condition but fails systematically on column communities across all three conditions, with mean column ARI remaining below 0.3 even at $r=0.5$; this is expected because SVDK does not account for node-specific popularity, and the singular vectors of $A$ mix community and popularity information in a way that K-means cannot reliably disentangle. The one exception to VEM's dominance is the power law condition at $r=0.1$, where the infinite-variance Pareto distribution concentrates the edge budget among a small number of hub nodes, leaving non-hub nodes too sparsely connected for the variational objective to escape poor initializations; as $r$ increases VEM recovers its advantage.
\begin{figure}[htp]
\centering
\includegraphics[width=1\linewidth]{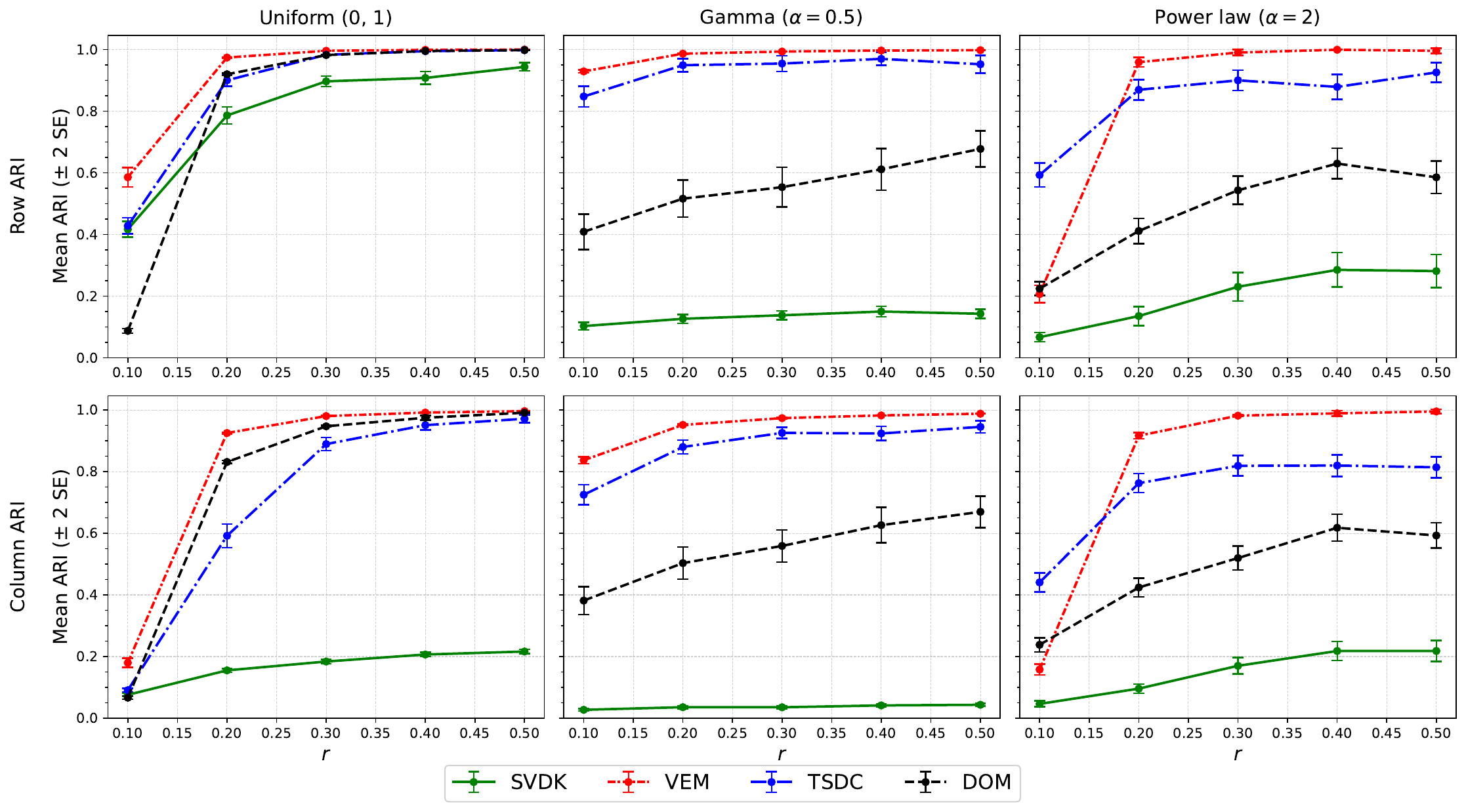}
\caption{Mean ARI ($\pm$ 2 SE) for row and column community recovery across three distributional conditions and five sparsity levels $r$.}\label{fig:clustering_accuracy}
\end{figure}

\subsection{Undirected Networks}
\label{sec:sim_symmetric}
The model is designed for biclustering in bipartite networks, and the proposed variational EM algorithm does not explicitly guarantee $\hat{q}^z = \hat{q}^w$ for symmetric adjacency matrices, although in principle, both $\hat{q}^z$ and $\hat{q}^w$ are expected to asymptotically converge to the same true labeling.
The iterative optimization of the variational EM algorithm is used without modification from the bipartite setting. To enforce a single symmetric partition in practice, we introduce a selection heuristic as a final post-processing step. Specifically, we evaluate the objective function $J$ twice: once using the estimated row parameters, $J(\hat{q}^z, \hat{q}^z, \hat{\pi}, \hat{\pi}, \hat{\theta}, \hat{\theta})$, and once using the estimated column parameters, $J(\hat{q}^w, \hat{q}^w, \hat{\rho}, \hat{\rho}, \hat{\lambda}, \hat{\lambda})$. The final community assignment is determined by selecting the configuration that yields the larger of these two objective values.

We adopt the simulation framework of \cite{sengupta2018block} with two communities $(K=2)$ and equal community sizes $(n_1 = n_2)$. The parameters are defined as $\lambda_{ir} = \alpha \sqrt{h/(1+h)}$ where $r=c_i$ and $\lambda_{ir} = \beta \sqrt{1/(1+h)}$ when $r \ne c_i$, where $h$ is the homophily factor controlling the ratio of intra- to inter-community edges, and $c_i$ denotes the true community label of the node $i$. Within each community, 50\% of the nodes are assigned as category 1  and the rest as category 2, with $\alpha = 0.8, \beta = 0.2$ for category 1 nodes and $\alpha = 0.2, \beta = 0.8$ for category 2 nodes, so that category 1 nodes are more popular within their own community and category 2 nodes more popular in the other. Edge counts are drawn from a Bernoulli distribution with mean $\lambda_{ic_j} \cdot \lambda_{jc_i}$. We evaluate two network sizes, $n \in \{400, 1000\}$, across $h \in \{1.5, 2.0, 2.5, 3.0, 3.5, 4.0\}$, with 100 replicates per condition. We compare against PABM \citep{sengupta2018block} and SSC \citep{norooziSSC2021}, both translated to Python from the original implementations.

\begin{figure}[htp]
    \centering
    \includegraphics[width=0.8\linewidth]{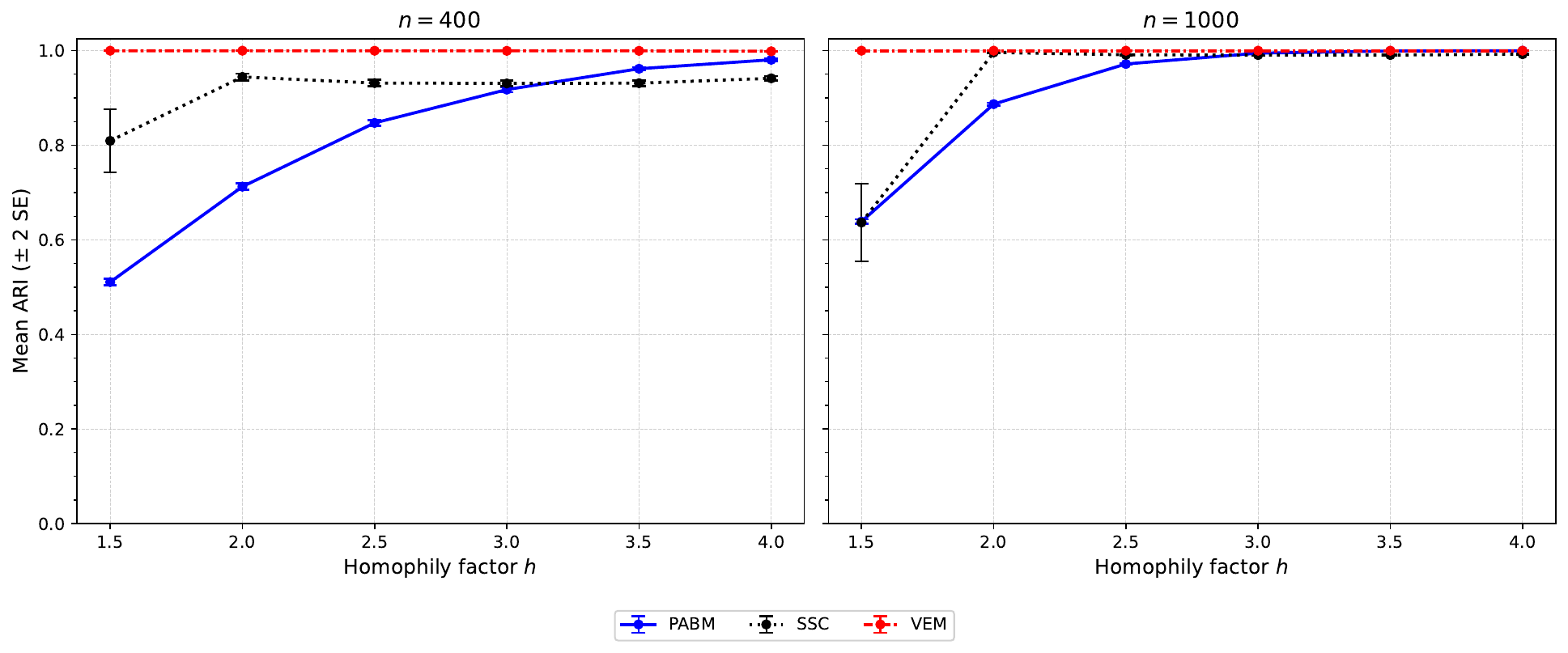}
    \caption{Mean ARI ($\pm$ 2 SE) across 6 values of homophily parameter $h$ and two network sizes $n$, averaged over 100 replicates.}
    \label{fig:sim_undirected}
\end{figure}

Figure \ref{fig:sim_undirected} reports the mean ARI ($\pm$ 2 SE). VEM achieves near-perfect recovery across the full range of $h$ at both network sizes, demonstrating that the proposed method adapts effectively to the undirected setting. PABM improves steadily with $h$ but lags behind VEM throughout, particularly at low homophily. SSC performs competitively with PABM at moderate to high homophily but exhibits higher variance at $h = 1.5$. All methods complete in a few seconds per replicate, consistent with runtimes observed in the bipartite setting.

\subsection{Computational Cost} \label{supp:computational_cost}

Table \ref{tab:computational_cost} reports the mean wall-clock runtime in seconds ($\pm$ 2 SE) across 100 replicates under the primary simulation setting described in Section \ref{sec:simulation1}. All simulations were executed on an AMD Milan processor running Linux.  DOM is the slowest method for the network sizes considered here, with runtimes ranging from roughly 1,600 to 10,000 seconds across conditions. SVDK is the fastest method, as expected for a one-shot spectral procedure. VEM and TSDC are broadly comparable, both completing in under 25 seconds, although VEM runs somewhat slower under the power-law condition, likely due to the larger number of iterations required for convergence under extreme degree heterogeneity. It is worth noting that the reported computational cost of our VEM procedure includes all random initializations, so the additional computational effort is fully accounted for in the timing comparisons. Overall, VEM and TSDC have comparable runtimes, whereas DOM is computationally expensive for these network sizes.

\begin{table}[htbp]
\centering
\caption{Mean wall-clock runtime in seconds ($\pm$ 2 SE) across 100 replicates.}
\label{tab:computational_cost}\par
\resizebox{\textwidth}{!}{%
\begin{tabular}{|llccc|}
\hline
 & Distribution & Uniform (0, 1) & Gamma ($\alpha=0.5$) & Power law ($\alpha=2$) \\
Method & $r$ &  &  &  \\
\hline
SVDK & $0.1$ & $0.41 \pm 0.20$ & $0.22 \pm 0.05$ & $0.19 \pm 0.03$ \\
 & $0.2$ & $0.24 \pm 0.05$ & $0.26 \pm 0.04$ & $0.21 \pm 0.04$ \\
 & $0.3$ & $0.26 \pm 0.03$ & $0.29 \pm 0.05$ & $0.21 \pm 0.03$ \\
 & $0.4$ & $0.22 \pm 0.02$ & $0.21 \pm 0.04$ & $0.24 \pm 0.04$ \\
 & $0.5$ & $0.23 \pm 0.01$ & $0.23 \pm 0.03$ & $0.20 \pm 0.03$ \\
\hline
VEM & $0.1$ & $13.14 \pm 0.80$ & $12.17 \pm 1.01$ & $17.14 \pm 2.02$ \\
 & $0.2$ & $11.04 \pm 1.08$ & $12.71 \pm 1.04$ & $18.99 \pm 2.22$ \\
 & $0.3$ & $11.38 \pm 0.92$ & $11.34 \pm 1.37$ & $20.13 \pm 2.52$ \\
 & $0.4$ & $8.41 \pm 1.00$ & $15.04 \pm 2.17$ & $17.83 \pm 2.99$ \\
 & $0.5$ & $6.62 \pm 0.87$ & $12.17 \pm 1.62$ & $18.77 \pm 2.75$ \\
\hline
TSDC & $0.1$ & $6.76 \pm 0.60$ & $113.48 \pm 211.85$ & $16.99 \pm 1.59$ \\
 & $0.2$ & $13.79 \pm 1.73$ & $5.39 \pm 1.11$ & $13.18 \pm 2.03$ \\
 & $0.3$ & $8.88 \pm 1.04$ & $4.88 \pm 0.90$ & $9.52 \pm 0.95$ \\
 & $0.4$ & $5.46 \pm 0.59$ & $4.72 \pm 0.87$ & $8.42 \pm 0.90$ \\
 & $0.5$ & $4.84 \pm 0.31$ & $4.40 \pm 0.56$ & $8.33 \pm 0.84$ \\
\hline
DOM & $0.1$ & $10219.51 \pm 1329.32$ & $5385.09 \pm 1290.24$ & $3559.69 \pm 325.59$ \\
 & $0.2$ & $3473.85 \pm 236.79$ & $5254.13 \pm 1477.36$ & $3668.80 \pm 291.08$ \\
 & $0.3$ & $2636.11 \pm 165.83$ & $6064.92 \pm 2424.88$ & $3586.85 \pm 333.95$ \\
 & $0.4$ & $1856.29 \pm 116.41$ & $4133.62 \pm 1474.06$ & $4008.37 \pm 598.16$ \\
 & $0.5$ & $1641.36 \pm 54.90$ & $3303.19 \pm 777.30$ & $3125.12 \pm 260.23$ \\
\hline
\hline
\end{tabular}
}
\end{table}

\section{Application to Real-World Networks}
\subsection{MovieLens Dataset}
In this section, we apply the proposed variational method to the popular MovieLens 100K dataset \citep{harper2015movielens}, collected by the GroupLens Research Project at the University of Minnesota through the MovieLens website (\url{https://movielens.unm.edu}). The dataset was gathered over a seven-month period from September 19, 1997, to April 22, 1998, and contains 100,000 ratings from 943 users on 1,682 movies. Each rating is an integer from 1 to 5.

Following prior studies such as \cite{flynn2020profile} (LBM) and \cite{zhao2024variational} (DC-LBM), we convert the ratings matrix into a binary format, forming an adjacency matrix $A \in \{0, 1\}^{943 \times 1682}$, where $A_{ij} = 1$ if user $i$ rated movie $j$, and $A_{ij} = 0$ otherwise. This representation captures the structure of user engagement rather than the specific rating values. Although we adopt a Poisson formulation for computational convenience, the theoretical arguments can be adapted to Bernoulli observations with only minor modifications, and in fact become simpler because the observations are bounded. The number of user clusters is fixed at $K = 3$ and the number of movie clusters at $L = 4$, consistent with the values used in \cite{flynn2020profile}. These values were originally determined based on scree plot visualizations of the profile likelihood under their model. We adopt the same settings in our algorithm to ensure a direct and fair comparison.


We assessed whether the movie cluster estimated by our algorithm align with the genre labels assigned to each movie. The 1682 movies in the dataset are annotated with one or more of 19 genre categories, such as Action, Romance, Drama, Animation, and others. Many movies belong to multiple genres, which makes direct comparison with the estimated clusters challenging. To address this, we followed the evaluation strategy employed by \cite{zhao2024variational} and filtered the dataset to retain only the 833 movies in the dataset that are labeled with a single genre, thus avoiding the ambiguity in the ground truth. This filtered subset allows us to construct a contingency table comparing the estimated movie clusters to the actual genre labels. We perform a chi-square test of independence to quantify the statistical association between the estimated clusters and true genre. We obtained a p-value of $8.947\times10^{-12}$ for the clusters estimated by the proposed variational EM algorithm, which is smaller than the p-value of 0.0415 from the latent block model (reported in \cite{zhao2024variational}), the p-value of $2.656\times10^{-7}$ from the degree-corrected latent block model (reported in \cite{zhao2024variational}), and the p-value of $4.39\times10^{-8}$
from DOM, and $3.003\times10^{-11}$ from TSDC (reported in \cite{jing2026two}). This indicates a stronger association with the known genre categories than previous approaches. These results suggest that the proposed method could be valuable for user segmentation and personalized recommendation.

\subsection{Undirected Networks}
As demonstrated by the simulation studies, the proposed algorithm, while originally designed for bipartite or directed networks, also performs effectively on undirected networks. We therefore are interested in evaluating its performance on real-world undirected networks. We evaluate our method on three widely studied network datasets: the Political Blogs network, the DBLP co-authorship network, and the British MPs network. The three networks have been analyzed in the paper on the undirected popularity-adjusted block model \citep{sengupta2018block}. The clustering performance of our proposed VEM method is compared against EPs-PABM (extreme Points algorithm)  used in \cite{sengupta2018block}.

The Political Blogs network \citep{adamic2005political} records directed hyperlinks between political blogs during the 2004 U.S. presidential election. Following common practice \citep{Karrer10, zhao2012consistency, amini2013pseudo}, we extract the largest connected component and treat the graph as undirected, since the primary interest is in clustering. The resulting network contains 1,222 nodes and 16,714 edges corresponding to the hyperlinks. The proposed VEM algorithm achieves the highest ARI of 0.8161, slightly outperforming both the regularized and unregularized EPs-PABM variants, which both achieve an ARI of 0.8101.

The second dataset is derived from DBLP network, a bibliography for computer science research. We use the heterogeneous version constructed in \cite{gao2009dblp} and \cite{ji2010dblp}, where the two authors are linked if they have co-authored a publication or attended the same conference. To get a setting comparable to previous studies, we focus on two major research areas, database and information retrieval, and extract the corresponding subgraph. The resulting network contains 2,203 nodes representing the authors and 1,148,044 edges reflecting co-authorship and shared conference attendance. On this dataset, the EPs-PABM variants achieve a higher clustering performance (0.8905). However, VEM remains highly competitive with an ARI of 0.8151.

The third dataset is the British MPs Twitter network \citep{greene2013}, which captures interactions between Members of Parliament. While the original data contains multiple edge layers and five political parties, we follow the preprocessing steps of \cite{sengupta2018block} by focusing strictly on the retweet interactions between members of the two major parties (Conservative and Labour). After extracting the largest connected component, the resulting undirected network contains 329 nodes and 5,720 edges. On this dataset, all methods demonstrate strong community recovery. VEM exactly matches the performance of the unregularized EPs-PABM with an ARI of 0.9878, while the regularized EPs-PABM achieves perfect recovery (1.0000). Overall, these comparisons demonstrate that the proposed VEM algorithm is highly competitive with  methods explicitly designed for undirected networks.

\section{Conclusion}
In this paper, we develop a variational expectation–maximization framework for estimating community labels in bipartite networks while accounting for node-specific popularity parameters. The proposed method provides scalable and accurate estimation, together with theoretical guarantees, including consistency of the estimated labels. A key challenge in the theoretical analysis is establishing the identifiability of community labels under the TNPM. To address this issue, we impose conditions on the link probabilities requiring that the link-probability vectors associated with any two distinct communities cannot be made identical through partitioning and scalar rescaling. We further refine these conditions to establish the well-separatedness of the maximizer. Extensive simulations and real-data analyses demonstrate that the proposed approach outperforms existing estimation methods across both bipartite and undirected network settings.

\section{Data Availability Statement}
The datasets are available at \\
\url{https://github.com/jonykarki/variational-estimators-node-popularity}.

\section*{Funding}
Yunpeng Zhao was supported by NSF Grant DMS-2245380.
\par

\bigskip
\begin{center}
{\large\bf SUPPLEMENTARY MATERIAL}
\end{center}

\begin{description}

\item Supplemental materials proofs of  technical results in the article and additional numerical studies. 

\end{description}

\bibliographystyle{apalike}

\bibliography{reference}

\clearpage
\setcounter{page}{1}
\renewcommand{\thepage}{S\arabic{page}}
\setcounter{section}{0}
\renewcommand{\thesection}{S\arabic{section}}
\renewcommand{\theHsection}{supp.\arabic{section}}
\setcounter{equation}{0}
\renewcommand{\theequation}{S.\arabic{equation}}
\renewcommand{\theHequation}{supp.\arabic{equation}}
\setcounter{table}{0}
\renewcommand{\thetable}{S\arabic{table}}
\renewcommand{\theHtable}{supp.\arabic{table}}
\setcounter{figure}{0}
\renewcommand{\thefigure}{S\arabic{figure}}
\renewcommand{\theHfigure}{supp.\arabic{figure}}
\setcounter{lemma}{0}
\renewcommand{\thelemma}{S.\arabic{lemma}}
\renewcommand{\theHlemma}{supp.\arabic{lemma}}
\setcounter{theorem}{0}
\renewcommand{\theHtheorem}{supp.\arabic{theorem}}
\setcounter{proposition}{0}
\renewcommand{\theHproposition}{supp.\arabic{proposition}}
\setcounter{definition}{0}
\renewcommand{\theHdefinition}{supp.\arabic{definition}}
\setcounter{remark}{0}
\renewcommand{\theHremark}{supp.\arabic{remark}}

\spacingset{1}
\begin{center}
  {\LARGE\bf Supplementary Materials for ``Variational Estimators for Node Popularity Models''}\\[1.5em]
  {\large Jony Karki \qquad Dongzhou Huang \qquad Yunpeng Zhao}\\[0.5em]
  {\large Department of Statistics, Colorado State University}
\end{center}
\bigskip

\spacingset{1.75}

\section{Proofs of technical results in Section 3}\label{appA}
\subsection{Proof of Proposition \ref{prop:mstep}}
\begin{proof}
Given $q_1$ and $q_2$, the term $q_1(z)q_2(w)\log(q_1(z)q_2(w))$ in the objective function $J(q_1,q_2,\Phi)$ is independent of $\Phi$. Therefore, it can be ignored when determining the global maximizer of $J$. The objective function may thus be rewritten in the following simplified form. With a slight abuse of notation, we continue to denote the simplified objective function by $J$.
\begin{align*}
    &\qquad J(q_1, q_2, \Phi) \\
    &= \sum_{z} \sum_{w} q_1(z) q_2(w) \text{log } P(z, w, A; \Phi) \\
    &= \sum_{z} \sum_{w} q_1(z) q_2(w) \Bigg[ - \sum_{i=1}^{m} \sum_{j=1}^{n} (\theta_{iw_j} \lambda_{jz_i}) \\
    &\quad + \sum_{i=1}^{m} \sum_{j=1}^{n} A_{ij} \text{log} (\theta_{iw_j} \lambda_{jz_i})  + \sum_{i=1}^{m} \text{log } \pi_{z_i} + \sum_{j=1}^{n} \text{log } \rho_{w_j} \Bigg] 
\end{align*}
\begin{align*}
    &= \sum_{z} \sum_{w} q_1(z) q_2(w) \Bigg[ - \sum_{i=1}^{m} \sum_{j=1}^{n} \sum_{k=1}^{K} \sum_{l=1}^{L} \mathds{1}(z_i=k) \mathds{1}(w_j=l) \theta_{il} \lambda_{jk} \\
    &\quad + \sum_{i=1}^{m} \sum_{j=1}^{n} A_{ij} \sum_{k=1}^{K} \sum_{l=1}^{L} \mathds{1}(z_i=k) \mathds{1}(w_j=l) \text{log} (\theta_{il} \lambda_{jk}) \\
    &\quad + \sum_{i=1}^{m} \sum_{k=1}^{K} \mathds{1}(z_i=k) \text{log } \pi_{k} + \sum_{j=1}^{n} \sum_{l=1}^{L} \mathds{1}(w_j=l) \text{log } \rho_{l} \Bigg] \\
    &= - \sum_{i=1}^{m} \sum_{j=1}^{n} \sum_{k=1}^{K} \sum_{l=1}^{L} P_{q_1}(z_i=k) P_{q_2}(w_j=l) \theta_{il} \lambda_{jk} \\
    & \quad + \sum_{i=1}^{m} \sum_{j=1}^{n} A_{ij} \sum_{k=1}^{K} \sum_{l=1}^{L} P_{q_1}(z_i=k) P_{q_2}(w_j=l) \text{log} (\theta_{il} \lambda_{jk}) \\
    &\quad + \sum_{i=1}^{m} \sum_{k=1}^{K} P_{q_1}(z_i=k) \text{log } \pi_{k} + \sum_{j=1}^{n} \sum_{l=1}^{L} P_{q_2}(w_j=l) \text{log } \rho_{l}.
\end{align*}
Observe that the objective function admits a decomposition into three separable components. Specifically, the first component depends only on $\lambda$ and $\theta$, the second depends only on $\pi$, and the third depends only on $\rho$. Consequently, the optimization problem decouples, and the global maximizers of $(\lambda, \theta, \pi, \rho)$ can be obtained by maximizing their respective components independently.

We begin with the second the component
\begin{equation*}
    \sum_{i=1}^{m} \sum_{k=1}^{K} P_{q_1}(z_i=k) \text{log } \pi_{k}.
\end{equation*}
It is straightforward to verify that the global maximizer of $\pi_k$ is given by
\begin{equation*}
    \hat{\pi}_k = \frac{ \sum_{i=1}^{m}  P_{q_1}(z_i=k)   }{m}.
\end{equation*}
Similarly, it follows from the third component that
\begin{equation*}
  \hat{\rho}_l = \frac{ \sum_{j=1}^{n} P_{q_2}(w_j=l) }{n}.
\end{equation*}

We now turn to the global maximizer of $( \theta, \lambda )$. For simplicity, let $\alpha_{il} = \log \theta_{il}$ and $\beta_{jk} = \log \lambda_{jk}$.
Then, the first component of $J$ can be written as follows:
\begin{align*}
&-\sum_{i=1}^{m} \sum_{j=1}^{n} \sum_{k=1}^{K} \sum_{l=1}^{L} P_{q_1}(z_i=k) P_{q_2}(w_j=l) \exp{(\alpha_{il} + \beta_{jk})} \\
&\quad + \sum_{i=1}^{m} \sum_{j=1}^{n} A_{ij} \sum_{k=1}^{K} \sum_{l=1}^{L} P_{q_1}(z_i=k) P_{q_2}(w_j=l) (\alpha_{il} + \beta_{jk}).
\end{align*}
With $q_1$ and $q_2$ fixed, the above function is concave in $(\alpha,\beta)$, since the exponential term $\exp(\alpha_{il}+\beta_{jk})$ is convex and the remaining term is linear. Consequently, the global maximizer can be characterized by the first-order optimality conditions. Taking partial derivatives with respect to $\alpha_{il}$ and $\beta_{jk}$ and setting them equal to zero, it follows after a rearrangement that
\begin{align*}
    \sum_{j=1}^{n} A_{ij} P_{q_2}(w_j=l) &= \exp (\alpha_{il}) \sum_{j=1}^{n} \sum_{k=1}^{K} P_{q_1}(z_i=k) P_{q_2}(w_j=l) \exp (\beta_{jk})  \\
    \sum_{i=1}^{m} A_{ij} P_{q_1}(z_i=k) &= \exp (\beta_{jk}) \sum_{i=1}^{m} \sum_{l=1}^{L} P_{q_1}(z_i=k) P_{q_2}(w_j=l) \exp (\alpha_{il}).
\end{align*}
Therefore, the global maximizer $( \hat{\alpha}_{il}, \hat{\beta}_{jk} )$ statisfies that
\begin{align*}
    \exp(\hat{\alpha}_{il}) &= \frac{\sum_{j=1}^{n} A_{ij} P_{q_2}(w_j=l)}{\sum_{j=1}^{n} \sum_{k=1}^{K} P_{q_1}(z_i=k) P_{q_2}(w_j=l) \exp (\hat{\beta}_{jk})} \\
    \exp(\hat{\beta}_{jk}) &= \frac{\sum_{i=1}^{m} A_{ij} P_{q_1}(z_i=k)}{\sum_{i=1}^{m} \sum_{l=1}^{L} P_{q_1}(z_i=k) P_{q_2}(w_j=l) \exp (\hat{\alpha}_{il})}.
\end{align*}
Noting that $ \hat{\theta}_{il} =  \exp(\hat{\alpha}_{il}) $ and $ \hat{\lambda}_{jk} = \exp(\hat{\beta}_{jk}) $, the desired result follows. We conclude the proof.
\end{proof}

\subsection{Proof of Proposition \ref{prop:mstep2}}
\begin{proof}
By the proof of Proposition \ref{prop:mstep}, it suffices to verify that $\hat{\theta}_{il}$ and $\hat{\lambda}_{jk}$, as defined in Proposition \ref{prop:mstep2}, satisfy the first-order optimality conditions for $J$ with respect to $\theta$ and $\lambda$ when both $q_1$ and $q_2$ are degenerate. Indeed,
\begin{align*}
    \pdv{J}{\theta_{il}} \bigg|_{\hat{\theta}, \hat{\lambda}} &= \frac{1}{\hat{\theta}_{il}} \sum_{j=1}^{n} A_{ij} P_{q_2}(w_j=l) -\sum_{j=1}^{n} \sum_{k=1}^{K} P_{q_1}(z_i=k) P_{q_2}(w_j=l) \hat{\lambda}_{jk} \\
    &= \frac{\sqrt{\sum_{j \in C_l} \sum_{i' \in R_{z_i}} A_{i'j}}}{\sum_{j \in C_l} A_{ij}}  \sum_{j \in C_l} A_{ij} - \sum_{j \in C_l} \sum_{k=1}^{K} P_{q_1}(z_i = k) \hat{\lambda}_{jk} \\
    &= \left( \sqrt{\sum_{j \in C_l} \sum_{i' \in R_{z_i}} A_{i'j}} \right) - \sum_{j \in C_l} \hat{\lambda}_{jz_i} \\
    &= \sqrt{\sum_{j \in C_l} \sum_{i' \in R_{z_i}} A_{i'j}} - \sum_{j \in C_l} \left( \frac{\sum_{i' \in R_{z_i}} A_{i'j}}{\sqrt{\sum_{i' \in R_{z_i}} \sum_{j' \in C_{w_{j}}} A_{i'j'}}} \right) \\
    &= \sqrt{\sum_{j \in C_l} \sum_{i' \in R_{z_i}} A_{i'j}} - \frac{\sum_{j \in C_l} \sum_{i' \in R_{z_i}} A_{i'j}}{\sqrt{\sum_{i' \in R_{z_i}} \sum_{j' \in C_l} A_{i'j'}}} \\
    &= 0.
\end{align*}
Similarly,
\begin{align*}
    \pdv{J}{\lambda_{jk}} \bigg|_{\hat{\theta}, \hat{\lambda}} &= \frac{1}{\hat{\lambda}_{jk}} \sum_{i=1}^{m} A_{ij} P_{q_1}(z_i=k) - \sum_{i=1}^{m} \sum_{l=1}^{L} P_{q_1}(z_i=k) P_{q_2}(w_j=l) \hat{\theta}_{il} \\
    &= \frac{\sqrt{\sum_{i \in R_k} \sum_{j' \in C_{w_j}} A_{ij'}}}{\sum_{i \in R_k} A_{ij}}  \sum_{i \in R_k} A_{ij} - \sum_{i \in R_k} \hat{\theta}_{iw_j} \\
    &= \sqrt{\sum_{i \in R_k} \sum_{j' \in C_{w_j}} A_{ij'}} - \sum_{i \in R_k} \left( \frac{\sum_{j' \in C_l} A_{ij'}}{\sqrt{\sum_{j' \in C_l} \sum_{i' \in R_{z_i}} A_{i'j'}}} \right) \\
    &= 0.
\end{align*}
We conclude the proof.
\end{proof}

\section{Proof of technical results in Section 4} \label{supp:proofs_sec4}

Before presenting the proofs of the main results---identifiability and label consistency---we first establish the concentration properties of the objective function. Recalling the definitions of $J(q_1, q_2, \Phi)$, $q^{z}_{ik}$, $q^{w}_{jl}$, it follows
\begin{align*}
&\qquad J(q_1, q_2, \Phi) \\
&= \sum_{z \in \Omega_z} \sum_{w \in \Omega_w} q_1(z) q_2(w) \Bigg[ \sum_{i=1}^{m} \log \pi_{z_i} + \sum_{j=1}^{n} \log \rho_{w_j} - \sum_{i=1}^{m} \sum_{j=1}^{n} \theta_{i w_j} \lambda_{j z_i} \\
&\quad + \sum_{i=1}^{m} \sum_{j=1}^{n} A_{ij} \log(\theta_{i w_j} \lambda_{j z_i}) \Bigg] - \sum_{z \in \Omega_z} \sum_{w \in \Omega_w} q_1(z) q_2(w) \log\bigl( q_1(z) q_2(w) \bigr), \\
&= -\sum_{i=1}^m \sum_{j=1}^n \left( \sum_{k=1}^K \sum_{l=1}^L q_{ik}^z q_{jl}^w \theta_{i l} \lambda_{j k} \right)  \\
&\qquad + \sum_{i=1}^m \sum_{j=1}^n A_{ij} \left( \sum_{k=1}^K \sum_{l=1}^L q_{ik}^z q_{jl}^w \log (\theta_{i l} \lambda_{j k}) \right) + \sum_{i=1}^m \sum_{k=1}^K q_{ik}^z \log \pi_k \\
&\qquad + \sum_{j=1}^n \sum_{l=1}^L q_{jl}^w \log \rho_l - \sum_{i=1}^m \sum_{k=1}
^K q_{ik}^z \log q_{ik}^z - \sum_{j=1}^n \sum_{l=1}^L q_{jl}^w \log q_{jl}^w.
\end{align*}
Since $q_1(z)$ and $q_2(w)$ have been represented by the matrices $q^z$ and $q^w$, we also use the notation $J(q^z, q^w, \Phi)$ to denote the above objective function.

Neglecting the lower-order terms and taking expectations of both sides of the above equation yields the corresponding population-level objective function:
\begin{align*}
 \bar{J}(q^z, q^w, \theta, \lambda) &= -\sum_{i=1}^m \sum_{j=1}^n \left( \sum_{k=1}^K \sum_{l=1}^L q_{ik}^z q_{jl}^w \theta_{i l} \lambda_{j k} \right) \\ 
 &\quad + \sum_{i=1}^m \sum_{j=1}^n E[A_{ij} \mid z^*, w^*] \left( \sum_{k=1}^K \sum_{l=1}^L q_{ik}^z q_{jl}^w \log (\theta_{i l} \lambda_{j k}) \right) \\
 &= -\sum_{i=1}^m \sum_{j=1}^n \left( \sum_{k=1}^K \sum_{l=1}^L q_{ik}^z q_{jl}^w \theta_{i l} \lambda_{j k} \right) \\ 
 &\quad + \sum_{i=1}^m \sum_{j=1}^n \theta_{i w_j^*}^* \lambda_{j z_i^*}^*  \left( \sum_{k=1}^K \sum_{l=1}^L q_{ik}^z q_{jl}^w \log (\theta_{i l} \lambda_{j k}) \right).
\end{align*}

We prove the following result connecting the sample and population versions of the objective functions.
\begin{theorem} \label{thm:concentration}
If $mn \tilde{\eta}^2 / ((m+n)(\log \tilde{\eta})^2) \rightarrow \infty$ as $m, n \rightarrow \infty$, then for any fixed positive constant $\epsilon$, we have
\begin{align*}
&\mathbb{P} \Bigg( \sup_{\substack{
q^z \in \mathcal{C}_z,\ q^w \in \mathcal{C}_w,\ \pi \in \mathcal{C}_\pi \\
\rho \in \mathcal{C}_\rho,\ \theta \in \mathcal{C}_\theta,\ \lambda \in \mathcal{C}_\lambda
}} \left| J(q^z, q^w, \Phi) - \bar{J}(q^z, q^w, \theta, \lambda) \right| \geq  \eta_{\max}^2 n\varepsilon \mid z^*, w^* \Bigg) \rightarrow 0.
\end{align*}
\end{theorem}
To prove this theorem, we need a  concentration inequality for sums involving products of Poisson random variables and deterministic coefficients \citep{Canonne_Poisson}. The relevant result, adapted from \cite{zhao2024variational}, is summarized in the following lemma.

\begin{lemma}  \label{lmcon:boundsumofPoisson}
Let $\{ X_{ij} \}$ be independent Poisson variables with mean $$E[X_{ij}] \leq r_{mn} C.$$ Then for all $\varepsilon >0$,
\begin{align*}
\mathbb{P}\left( \max_{ \substack{ 0 \leq u_i \leq 1, i=1, \dots, m \\ 0 \leq v_j \leq 1, j=1, \dots, n} } \left|  \sum_{i=1}^{m} \sum_{j=1}^{n} \left( X_{ij} - E[X_{ij}] \right) u_i v_j \right| \geq mn r_{mn} \varepsilon \right) \\
\leq 2^{m+n+1} \exp\left( - \frac{ mnr_{mn} \varepsilon^2 }{4 \max( C, \varepsilon )} \right). 
\end{align*}
\end{lemma}
\begin{proof}
Note that
\begin{align*}
&\quad J(q^z, q^w, \Phi) - \bar{J}(q^z, q^w, \theta, \lambda) \\ &= \sum_{i=1}^m \sum_{j=1}^n \left(A_{ij} - E\left[ A_{ij} \mid z^{*}, w^{*} \right]\right) \left( \sum_{k=1}^K \sum_{l=1}^L q_{ik}^z q_{jl}^w \log (\theta_{i l} \lambda_{j k}) \right)   \\
& - \sum_{i=1}^m \sum_{k=1}^K q_{ik}^z \log (q_{ik}^z/ \pi_k) - \sum_{j=1}^n \sum_{l=1}^L q_{jl}^w \log (q_{jl}^w / \rho_l). 
\end{align*}
Therefore,
\begin{align}
& \sup_{q^z \in \mathcal{C}_z,\, q^w \in \mathcal{C}_w,\, \pi \in \mathcal{C}_\pi,\, \rho \in \mathcal{C}_\rho,\, \theta \in \mathcal{C}_\theta,\, \lambda \in \mathcal{C}_\lambda} \left| J(q^z, q^w, \Phi) - \bar{J}(q^z, q^w, \theta, \lambda) \right| \nonumber \\
&\leq \sup_{\substack{ q^{z} \in \mathcal{C}_{z},\ q^{w} \in \mathcal{C}_{w} \\
\theta \in \mathcal{C}_{\theta},\ \lambda \in \mathcal{C}_{\lambda}}} \left|  \sum_{i=1}^m \sum_{j=1}^n \left(A_{ij} - E\left[ A_{ij} \mid z^{*}, w^{*} \right]\right) \left( \sum_{k=1}^K \sum_{l=1}^L q_{ik}^z q_{jl}^w \log (\theta_{i l} \lambda_{j k}) \right)  \right| \nonumber \\
& + \sup_{q^{z} \in \mathcal{C}_{z}, \pi \in \mathcal{C}_{\pi} } \left| \sum_{i=1}^m \sum_{k=1}^K q_{ik}^z \log (q_{ik}^z/ \pi_k)  \right| 
+ \sup_{q^{w} \in \mathcal{C}_{w}, \rho \in \mathcal{C}_{\rho} } \left| \sum_{j=1}^n \sum_{l=1}^L q_{jl}^w \log (q_{jl}^w / \rho_l)  \right|.  \label{eqcon:mainobjectivefunction}
\end{align}
We first estimate the first term on the right-hand side of \eqref{eqcon:mainobjectivefunction}. It follows by a straightforward calculation that
\begin{align*}
&\quad \sum_{i=1}^m \sum_{j=1}^n \left(A_{ij} - E\left[ A_{ij} \mid z^{*}, w^{*} \right]\right) \left( \sum_{k=1}^K \sum_{l=1}^L q_{ik}^z q_{jl}^w \log (\theta_{i l} \lambda_{j k}) \right) \\
&= \sum_{i=1}^m \sum_{j=1}^n \left(A_{ij} - E\left[ A_{ij} \mid z^{*}, w^{*} \right]\right) \left( \sum_{k=1}^K \sum_{l=1}^L q_{ik}^z q_{jl}^w \left(\log \theta_{i l} + \log \lambda_{j k}\right) \right) \\
&= \sum_{i=1}^m \sum_{j=1}^n \left(A_{ij} - E\left[ A_{ij} \mid z^{*}, w^{*} \right]\right) \\
&\qquad \left(\sum_{k=1}^{K} q_{ik}^{z} \sum_{l=1}^{L} q_{jl}^{w} \log \theta_{il} + \sum_{k=1}^{K} q_{ik}^{z} \log \lambda_{jk} \sum_{l=1}^{L} q_{jl}^{w}\right) \\
&= \sum_{i=1}^m \sum_{j=1}^n \left(A_{ij} - E\left[ A_{ij} \mid z^{*}, w^{*} \right]\right) \left(  \sum_{l=1}^{L} q_{jl}^{w} \log \theta_{il} + \sum_{k=1}^{K} q_{ik}^{z} \log \lambda_{jk} \right),
\end{align*}
where the last equality follows by the settings that $\sum_{k} q_{ik}^{z} =1$ and $\sum_{l} q_{jl}^{w} =1$. Furthermore, 
\begin{eqnarray*}
&& \sum_{i=1}^m \sum_{j=1}^n \left(A_{ij} - E\left[ A_{ij} \mid z^{*}, w^{*} \right]\right) \sum_{l=1}^{L} q_{jl}^{w} \log \theta_{il}  \\
&=& \sum_{i=1}^m \sum_{j=1}^n \left(A_{ij} - E\left[ A_{ij} \mid z^{*}, w^{*} \right]\right) \times \\
&&\qquad \sum_{l=1}^{L} q_{jl}^{w} \left(\frac{ \log(\theta_{il}/ \eta_{\mathrm{min}}) }{\log(\eta_{\mathrm{max}}/\eta_{\mathrm{min}})} + \frac{ \log( \eta_{\mathrm{min}}) }{\log(\eta_{\mathrm{max}}/\eta_{\mathrm{min}})}\right) \times \log(\eta_{\mathrm{max}}/\eta_{\mathrm{min}}) \\
&=& \log(\eta_{\mathrm{max}}/\eta_{\mathrm{min}}) \sum_{i=1}^m \sum_{j=1}^n \left(A_{ij} - E\left[ A_{ij} \mid z^{*}, w^{*} \right]\right)  \sum_{l=1}^{L} q_{jl}^{w}  \times \frac{ \log(\theta_{il}/ \eta_{\mathrm{min}}) }{\log(\eta_{\mathrm{max}}/\eta_{\mathrm{min}})} \\
&& + \log(\eta_{\mathrm{min}}) \sum_{i=1}^m \sum_{j=1}^n \left(A_{ij} - E\left[ A_{ij} \mid z^{*}, w^{*} \right]\right)  \sum_{l=1}^{L} q_{jl}^{w} \\
&=& \log(\eta_{\mathrm{max}}/\eta_{\mathrm{min}}) \sum_{l=1}^{L} \sum_{i=1}^m \sum_{j=1}^n \left(A_{ij} - E\left[ A_{ij} \mid z^{*}, w^{*} \right]\right)   q_{jl}^{w}  \times \frac{ \log(\theta_{il}/ \eta_{\mathrm{min}}) }{\log(\eta_{\mathrm{max}}/\eta_{\mathrm{min}})} \\
&& + \log(\eta_{\mathrm{min}}) \sum_{i=1}^m \sum_{j=1}^n \left(A_{ij} - E\left[ A_{ij} \mid z^{*}, w^{*} \right]\right),
\end{eqnarray*}
where the last equality again uses the fact that $\sum_{l} q_{jl}^{w} = 1$. Similarly,
\begin{eqnarray*}
&& \sum_{i=1}^m \sum_{j=1}^n \left(A_{ij} - E\left[ A_{ij} \mid z^{*}, w^{*} \right]\right) \sum_{k=1}^{K} q_{ik}^{z} \log \lambda_{jk} \\
&=& \log( \eta_{\mathrm{max}} / \eta_{\mathrm{min}} ) \sum_{k=1}^{K} \sum_{i=1}^m \sum_{j=1}^n \left(A_{ij} - E\left[ A_{ij} \mid z^{*}, w^{*} \right]\right)   q_{ik}^{z}  \times \frac{ \log(\lambda_{jk}/ \eta_{\mathrm{min}}) }{\log(\eta_{\mathrm{max}}/\eta_{\mathrm{min}})} \\ 
&&+ \log(\eta_{\mathrm{min}}) \sum_{i=1}^m \sum_{j=1}^n \left(A_{ij} - E\left[ A_{ij} \mid z^{*}, w^{*} \right]\right).
\end{eqnarray*}
Combining the last three displays yields
\begin{align*}
&\quad \sum_{i=1}^m \sum_{j=1}^n \left(A_{ij} - E\left[ A_{ij} \mid z^{*}, w^{*} \right]\right) \left( \sum_{k=1}^K \sum_{l=1}^L q_{ik}^z q_{jl}^w \log (\theta_{i l} \lambda_{j k}) \right) \\
&= \log(\eta_{\mathrm{max}}/\eta_{\mathrm{min}}) \sum_{l=1}^{L} \sum_{i=1}^m \sum_{j=1}^n \left(A_{ij} - E\left[ A_{ij} \mid z^{*}, w^{*} \right]\right)   q_{jl}^{w}  \times \frac{ \log(\theta_{il}/ \eta_{\mathrm{min}}) }{\log(\eta_{\mathrm{max}}/\eta_{\mathrm{min}})} \\
& + \log( \eta_{\mathrm{max}} / \eta_{\mathrm{min}} ) \sum_{k=1}^{K} \sum_{i=1}^m \sum_{j=1}^n \left(A_{ij} - E\left[ A_{ij} \mid z^{*}, w^{*} \right]\right)   q_{ik}^{z}  \times \frac{ \log(\lambda_{jk}/ \eta_{\mathrm{min}}) }{\log(\eta_{\mathrm{max}}/\eta_{\mathrm{min}})} \\ 
& + \left(\log(\eta_{\mathrm{min}}) +  \log(\eta_{\mathrm{min}}) \right) \sum_{i=1}^m \sum_{j=1}^n \left(A_{ij} - E\left[ A_{ij} \mid z^{*}, w^{*} \right]\right).
\end{align*}
Note that for $q^{w} \in \mathcal{C}_{w}$ and $\theta \in \mathcal{C}_{\theta}$, we have $q_{jl}^{w} \in [0,1]$ for all $j,l$ \newline and $\log(\theta_{il}/ \eta_{\mathrm{min}}) / \log ( \eta_{\mathrm{max}} / \eta_{\mathrm{min}} ) \in [0,1]$ for all $i,l$. Therefore, for $l =1, \dots, L$,
\begin{eqnarray*}
&& \sup_{q^{w} \in \mathcal{C}_{w},  \theta \in \mathcal{C}_{\theta} } \left|  \sum_{i=1}^m \sum_{j=1}^n \left(A_{ij} - E\left[ A_{ij} \mid z^{*}, w^{*} \right]\right)    q_{jl}^{w}  \times \frac{ \log(\theta_{il}/ \eta_{\mathrm{min}}) }{\log(\eta_{\mathrm{max}}/\eta_{\mathrm{min}})} \right|  \\
&\leq& \max_{ \substack{ 0 \leq u_i \leq 1, i=1, \dots, m \\ 0 \leq v_j \leq 1, j=1, \dots, n} } \left|  \sum_{i=1}^{m} \sum_{j=1}^{n} \left(A_{ij} - E\left[ A_{ij} \mid z^{*}, w^{*} \right]\right) u_i v_j \right|.
\end{eqnarray*}
Similarly, it follows
\begin{eqnarray*}
&& \sup_{q^{z} \in \mathcal{C}_{z},  \lambda \in \mathcal{C}_{\lambda} } \left|  \sum_{i=1}^m \sum_{j=1}^n \left(A_{ij} - E\left[ A_{ij} \mid z^{*}, w^{*} \right]\right)    q_{ik}^{z}  \times \frac{ \log(\lambda_{jk}/ \eta_{\mathrm{min}}) }{\log(\eta_{\mathrm{max}}/\eta_{\mathrm{min}})} \right|  \\
&\leq& \max_{ \substack{ 0 \leq u_i \leq 1, i=1, \dots, m \\ 0 \leq v_j \leq 1, j=1, \dots, n} } \left|  \sum_{i=1}^{m} \sum_{j=1}^{n} \left(A_{ij} - E\left[ A_{ij} \mid z^{*}, w^{*} \right]\right) u_i v_j \right|.
\end{eqnarray*}
Furthermore, it is straightforward that
\begin{align*}
&\sum_{i=1}^m \sum_{j=1}^n \left(A_{ij} - E\left[ A_{ij} \mid z^{*}, w^{*} \right]\right) \\ 
&\quad \leq \max_{ \substack{ 0 \leq u_i \leq 1, i=1, \dots, m \\ 0 \leq v_j \leq 1, j=1, \dots, n} } \left|  \sum_{i=1}^{m} \sum_{j=1}^{n} \left(A_{ij} - E\left[ A_{ij} \mid z^{*}, w^{*} \right]\right) u_i v_j \right|.
\end{align*}
Combining the last four displays, it follows after a rearrangement that
\begin{align*}
\sup_{\substack{
q^{z} \in \mathcal{C}_{z},\ q^{w} \in \mathcal{C}_{w} \\
\theta \in \mathcal{C}_{\theta},\ \lambda \in \mathcal{C}_{\lambda}
}} \left|  \sum_{i=1}^m \sum_{j=1}^n \left(A_{ij} - E\left[ A_{ij} \mid z^{*}, w^{*} \right]\right) \left( \sum_{k=1}^K \sum_{l=1}^L q_{ik}^z q_{jl}^w \log (\theta_{i l} \lambda_{j k}) \right)  \right| \\
\leq  C \max_{ \substack{ 0 \leq u_i \leq 1, i=1, \dots, m \\ 0 \leq v_j \leq 1, j=1, \dots, n} } \left|  \sum_{i=1}^{m} \sum_{j=1}^{n} \left(A_{ij} - E\left[ A_{ij} \mid z^{*}, w^{*} \right]\right) u_i v_j \right|,
\end{align*}
where
\begin{equation*}
 C :=  \left | \log(\eta_{\mathrm{max}}^{K+L}/\eta_{\mathrm{min}}^{K+L-2}) \right |.
\end{equation*}
Furthermore, given $z^{*}, w^{*}$, $\{A_{ij}\}$ is a sequence of Poisson variables with the conditional mean $E[A_{ij} \mid z^{*}, w^{*} ] = \theta_{i w^{*}_{j}}^{*}  \lambda_{j z^{*}_{i}}^{*} \leq \eta_{\mathrm{max}}^2$. Applying Lemma \ref{lmcon:boundsumofPoisson}, it follows after a rearrangement that
\begin{align*}
\mathbb{P}\left(   C  \max_{ \substack{ 0 \leq u_i \leq 1 \\ 0 \leq v_j \leq 1} } \left|  \sum_{i=1}^{m} \sum_{j=1}^{n} \left(A_{ij} - E\left[ A_{ij} \mid z^{*}, w^{*} \right]\right) u_i v_j \right| \geq  \eta_{\max}^2 mn\varepsilon/2  \, \Bigg| \, z^{*}, w^{*} \right) \\
\leq 2^{m+n+1} \exp\left( - \frac{ mn \eta_{\max}^2 (\varepsilon/2C)^2 }{ 4\max\left( 1  , \varepsilon/(2C) \right) } \right) = 2^{m+n+1} \exp\left( - \frac{ mn \eta_{\max}^2 \varepsilon^2 }{ 16 C^2  } \right) ,
\end{align*}
and thus,
\begin{align}
& \mathbb{P} \Bigg(\sup_{\substack{
q^{z} \in \mathcal{C}_{z},\ q^{w} \in \mathcal{C}_{w} \\
\theta \in \mathcal{C}_{\theta},\ \lambda \in \mathcal{C}_{\lambda}
}} \left|  \sum_{i=1}^m \sum_{j=1}^n \left(A_{ij} - E\left[ A_{ij} \mid z^{*}, w^{*} \right]\right) \left( \sum_{k=1}^K \sum_{l=1}^L q_{ik}^z q_{jl}^w \log (\theta_{i l} \lambda_{j k}) \right)  \right| \nonumber \\
&\qquad \geq mn\varepsilon/2  \, \Bigg| \, z^{*}, w^{*}  \Bigg) \leq 2^{m+n+1} \exp\left( - \frac{ mn \eta_{\max}^2 \varepsilon^2 }{ 16 C^2  } \right).\label{eqcon:firstbound1}
\end{align}
We now turn to the second term on the right-hand side of \eqref{eqcon:mainobjectivefunction}. By a standard result of Kullback–Leibler divergence, we obtain for all $ q^{z} \in \mathcal{C}_z $ and $\pi \in \mathcal{C}_{\pi}$
\begin{equation*}
\sum_{i=1}^m \sum_{k=1}^K q_{ik}^z \log (q_{ik}^z/ \pi_k) \geq 0.
\end{equation*}
Furthermore, applying Jensen's inequality, it follows that for $ q^{z} \in \mathcal{C}_z $ and $\pi \in \mathcal{C}_{\pi}$
\begin{align*}
&\quad \sum_{i=1}^m \sum_{k=1}^K q_{ik}^z \log (q_{ik}^z/ \pi_k) \\
&\leq \sum_{i=1}^m \log\left( \sum_{k=1}^K q_{ik}^z \times q_{ik}^z/ \pi_k \right) \leq  \sum_{i=1}^m \log\left( \sum_{k=1}^K \left(q_{ik}^z\right)^2 / \pi_{\mathrm{min}} \right) \\
& \leq \sum_{i=1}^m \log\left( \frac{1}{ \pi_{\mathrm{min}}}  \left( \sum_{k=1}^K q_{ik}^z\right)^2   \right) = \sum_{i=1}^m \log  \left(\pi_{\mathrm{min}}^{-1}\right) = m \log \left( \pi_{\mathrm{min}}^{-1}\right),
\end{align*}
where in the last second equality, we invoked the fact that $\sum_{k} q_{ik}^z =1$. Combining the preceding two results, we have
\begin{equation*}
\sup_{q^{z} \in \mathcal{C}_{z}, \pi \in \mathcal{C}_{\pi} } \left| \sum_{i=1}^m \sum_{k=1}^K q_{ik}^z \log (q_{ik}^z/ \pi_k)  \right|  \leq m \log  \left(\pi_{\mathrm{min}}^{-1}\right).
\end{equation*}
Then, it follows directly that
\begin{equation}
\mathbb{P}\left( \sup_{\substack{
q^{z} \in \mathcal{C}_{z} \\
\pi \in \mathcal{C}_{\pi}
}} \left| \sum_{i=1}^m \sum_{k=1}^K q_{ik}^z \log (q_{ik}^z/ \pi_k)  \right|  \geq mn\varepsilon/4 \, \Bigg| \, z^{*}, w^{*} \right) \leq  \mathds{1}_{ \left\{ n : n \varepsilon \leq 4 \log  \left(\pi_{\mathrm{min}}^{-1}\right)  \right\} }. \label{eqcon:firstbound2}
\end{equation}
Similarly, we have
\begin{equation}
\mathbb{P}\left(  \sup_{\substack{
q^{w} \in \mathcal{C}_{w} \\
\rho \in \mathcal{C}_{\rho}
}} \left| \sum_{j=1}^n \sum_{l=1}^L q_{jl}^w \log (q_{jl}^w / \rho_l)  \right| \geq mn\varepsilon/4 \, \Bigg| \, z^{*}, w^{*} \right) \leq \mathds{1}_{ \left\{ m : m \varepsilon \leq 4 \log  \left(\rho_{\mathrm{min}}^{-1}\right)  \right\} }. \label{eqcon:firstbound3}
\end{equation}

Combining \eqref{eqcon:mainobjectivefunction} with \eqref{eqcon:firstbound1}, \eqref{eqcon:firstbound2}, and \eqref{eqcon:firstbound3}, and noting that $C \asymp |\log \tilde{\theta}|$, the desired result follows.
\end{proof}
\begin{proposition}
\label{prop:maximizer}
    For all $q^z \in \mathcal{C}_z, q^w \in \mathcal{C}_w,$ $\theta \in \mathcal{C}_\theta,$ and $\lambda \in \mathcal{C}_\lambda$, we have,
\begin{align*}
    \bar{J}(1^{z^*}, 1^{w^*}, \theta^*, \lambda^*) - \bar{J}(q^z, q^w, \theta, \lambda) \ge 0.
\end{align*}
\end{proposition}
\begin{proof}
\begin{align*}
&\quad \bar{J}_1(q^z, q^w, \theta, \lambda) \\
&= \sum_{i=1}^{m} \sum_{j=1}^{n} \theta_{iw_j^*}^* \lambda_{jz_i^*}^* \left( \sum_{k'=1}^{K} \sum_{l'=1}^{L} q_{ik'}^z q_{jl'}^w \log (\theta_{i l'} \lambda_{j k'}) \right) \\
&= \sum_{i=1}^{m} \sum_{j=1}^{n} \left( \sum_{k=1}^{K} \sum_{l=1}^{L} 1_{i k}^{z^*} 1_{j l}^{w^*} \theta_{i l}^* \lambda_{j k}^* \right) \left( \sum_{k'=1}^{K} \sum_{l'=1}^{L} q_{ik'}^z q_{jl'}^w \log (\theta_{i l'} \lambda_{j k'}) \right) \\
&= \sum_{k=1}^{K} \sum_{l=1}^{L} \sum_{k'=1}^{K} \sum_{l'=1}^{L} \sum_{i=1}^{m} \sum_{j=1}^{n} 1_{i k}^{z^*} 1_{j l}^{w^*} q_{ik'}^z q_{jl'}^w \theta_{i l}^* \lambda_{j k}^* \log (\theta_{i l'} \lambda_{j k'}).
\end{align*}
\begin{align*}
&\quad \bar{J}_1 ( 1^{z^*}, 1^{w^*}, \theta^*, \lambda^*) \\
&= \sum_{i=1}^{m} \sum_{j=1}^{n} \theta_{iw_j^*}^* \lambda_{jz_i^*}^* \left( \sum_{k=1}^{K} \sum_{l=1}^{L} 1_{ik}^{z^*} 1_{jl}^{w^*} \log \left( \theta_{il}^* \lambda_{jk}^* \right) \right)  \left( \sum_{k'=1}^{K} \sum_{l'=1}^{L} q_{ik'}^{z} q_{jl'}^{w} \right) \\
&= \sum_{i=1}^{m} \sum_{j=1}^{n} \left( \sum_{k=1}^{K} \sum_{l=1}^{L} 1_{ik}^{z^*} 1_{jl}^{w^*} \theta_{il}^* \lambda_{jk}^* \log \left( \theta_{il}^* \lambda_{jk}^* \right) \right) \left( \sum_{k'=1}^{K} \sum_{l'=1}^{L} q_{ik'}^{z} q_{jl'}^{w} \right) \\
&= \sum_{k=1}^{K} \sum_{l=1}^{L} \sum_{k'=1}^{K} \sum_{l'=1}^{L} \sum_{i=1}^{m} \sum_{j=1}^{n} 1_{ik}^{z^*} 1_{jl}^{w^*} q_{ik'}^{z} q_{jl'}^{w} \theta_{il}^* \lambda_{jk}^* \log \left( \theta_{il}^* \lambda_{jk}^* \right) .
\end{align*}
\begin{align*}
&\quad\bar{J}_2 ( q^z, q^w, \theta, \lambda) \\
&= - \sum_{i=1}^{m} \sum_{j=1}^{n} \left( \sum_{k'=1}^{K} \sum_{l'=1}^{L} q^z_{ik'} q^w_{jl'} \theta_{il'} \lambda_{jk'} \right) \\
&= - \sum_{i=1}^{m} \sum_{j=1}^{n} \left( \sum_{k=1}^{K} \sum_{l=1}^{L} 1^{z^*}_{ik} 1^{w^*}_{jl} \right) \left( \sum_{k'=1}^{K} \sum_{l'=1}^{L} q^z_{ik'} q^w_{jl'} \theta_{il'} \lambda_{jk'} \right) \\
&= - \sum_{k=1}^{K} \sum_{l=1}^{L} \sum_{k'=1}^{K} \sum_{l'=1}^{L} \sum_{i=1}^{m} \sum_{j=1}^{n} 1^{z^*}_{ik} 1^{w^*}_{jl} q^z_{ik'} q^w_{jl'} \theta_{il'} \lambda_{jk'} .
\end{align*}
\begin{align*}
&\quad \bar{J}_2 ( 1^{{z}^*}, 1^{{w}^*}, \theta^*, \lambda^*) \\
&= - \sum_{i=1}^{m} \sum_{j=1}^{n} \left( \sum_{k=1}^{K} \sum_{l=1}^{L} 1_{ik}^{z^*} 1_{jl}^{w^*} \theta_{il}^* \lambda^*_{jk} \right) \left( \sum_{k'=1}^{K} \sum_{l'=1}^{L} q^{z}_{ik'} q^{w}_{jl'} \right) \\
&= - \sum_{k=1}^{K} \sum_{\ell=1}^{L} \sum_{k'=1}^{K} \sum_{\ell'=1}^{L} \sum_{i=1}^{m} \sum_{j=1}^{n} 1_{ik}^{z^*} 1_{jl}^{w^*} q^{z}_{ik'} q^{w}_{jl'} \theta^*_{il} \lambda^*_{jk}.
\end{align*}
Putting the last four equations together, we have
\begin{align}
&\bar{J} \left( 1^{z^*}, 1^{w^*}, \theta^*, \lambda^* \right) - \bar{J} \left( q^z, q^w, \theta, \lambda \right)  \nonumber \\
&= \left[ \sum_{k=1}^K \sum_{l=1}^L \sum_{k'=1}^K \sum_{l'=1}^L \sum_{i=1}^m \sum_{j=1}^n 1_{ik}^{z^*} 1_{jl}^{w^*} q_{ik'}^z q_{jl'}^w \theta_{il}^* \lambda_{jk}^* \left( \log \left( \theta_{il}^* \lambda_{jk}^* \right) - \log \left( \theta_{il'} \lambda_{jk'} \right) \right) \right] \nonumber \\
&\qquad + \left[ \sum_{k=1}^K \sum_{l=1}^L \sum_{k'=1}^K \sum_{l'=1}^L \sum_{i=1}^m \sum_{j=1}^n 1_{ik}^{z^*} 1_{jl}^{w^*} q_{il'}^z q_{jl'}^w \left( -\theta_{il}^* \lambda_{jk}^* + \theta_{il'} \lambda_{jk'} \right) \right]  \nonumber \\
&= \sum_{k=1}^K \sum_{l=1}^L \sum_{k'=1}^K \sum_{l'=1}^L \sum_{i=1}^m \sum_{j=1}^n 1_{ik}^{z^*} 1_{jl}^{w^*} q_{ik'}^z q_{jl'}^w \\
&\qquad \left[ \theta_{il}^* \lambda_{jk}^* \left( \log \left( \theta_{il}^* \lambda_{jk}^* \right) - \log \left( \theta_{il'} \lambda_{jk'} \right) \right) - \left( \theta_{il}^* \lambda_{jk}^* - \theta_{il'} \lambda_{jk'} \right) \right] \nonumber \\
&= \sum_{k=1}^K \sum_{l=1}^L \sum_{k'=1}^K \sum_{l'=1}^L \sum_{i=1}^m \sum_{j=1}^n 1_{ik}^{z^*} 1_{jl}^{w^*} q_{ik'}^z q_{jl'}^w \text{ KL}(\theta_{il}^* \lambda_{jk}^*, \theta_{il'} \lambda_{jk'}) \ge 0,  \label{eqm:formulafordifferenceJ}
\end{align}
Where $\mathrm{KL}(a,b)$ is defined as $ a \log (a/b) - (a -b) $, for all $a, b >0$. It follows from Lemma 16 in \cite{zhao2024variational} that $\mathrm{KL}(a,b) \geq 0$. Thus, we conclude the proof.

\end{proof}

The proof of Theorem \ref{thm:identification} relies on the following key lemma, which allows us to focus solely on the case where $q^{z} \in \mathcal{I}_{z}$ and $q^{w} \in \mathcal{I}_{w}$.
\begin{lemma}  \label{lm:restrict}
Under Condition \ref{cond:c1}, the following two statements are equivalent:
\begin{enumerate}[label=(\roman*)]
\item For all $q^{z} \in \mathcal{C}_{z}$ and $q^{w} \in \mathcal{C}_{w}$ such that $M_{\mathrm{row}} ( q^{z}) >0 $ or $ M_{ \mathrm{col} } (q^{w}) >0 $, we have $\bar{J}(1^{z^*}, 1^{w^*}, \theta^*, \lambda^*) - \bar{J}(q^z, q^w, \theta, \lambda) > 0$.
\item For all $1^{z} \in \mathcal{I}_{z}$ and $1^{w} \in \mathcal{I}_{w}$ such that $M_{\mathrm{row}} ( 1^{z}) >0 $ or $ M_{ \mathrm{col} } (1^{w}) >0 $, we have $\bar{J}(1^{z^*}, 1^{w^*}, \theta^*, \lambda^*) - \bar{J}(1^z, 1^w, \theta, \lambda) > 0$.
\end{enumerate}
\end{lemma}

\begin{proof}
It is obvious that (i) implies (ii). Thus, it suffices to prove that (ii) implies (i).

For any $q^{z} \in \mathcal{C}_{z}$ and $q^{w} \in \mathcal{C}_{w}$ such that $M_{\mathrm{row}} ( q^{z}) >0 $ or $ M_{ \mathrm{col} } (q^{w}) >0 $. Our goal is to show that
\begin{equation*}
\bar{J}(1^{z^*}, 1^{w^*}, \theta^*, \lambda^*) - \bar{J}(q^z, q^w, \theta, \lambda) > 0.
\end{equation*}
Due to the symmetry of the argument, it suffices to prove the inequality in the case where $ M_{\mathrm{row}} ( q^{z}) >0 $.
We begin by showing that there exists $ 1^{ \tilde{z} } \in \mathcal{I}_{z}(q^{z}) $ such that $ M_{ \mathrm{row} } ( 1^{ \tilde{z} } ) >0 $. Assume that this does not hold, then for all $ 1^{z} \in \mathcal{I}_{z}(q^{z}) $, $ M_{ \mathrm{row} } ( 1^{ z } ) =0 $. Given any $ 1^{ z' } \in \mathcal{I}_{z}(q^{z}) $
, we claim that there exist $ \tilde{i}, \tilde{k} $ such that
\begin{equation*}
q^{z}_{ \tilde{i} \tilde{k} } >0 \quad \text{but} \quad 1^{z'}_{  \tilde{i} \tilde{k} } =0.
\end{equation*}
Otherwise, it follows by the definition of $\mathcal{I}_{z}(q^{z})$ that $q^{z}_{ik} =0$ for all $i,k$ satisfying $ 1^{z'}_{ik}=0 $. As a result, $ q^{z} = 1^{z'} $. Noting that $ 1^{ z' } \in \mathcal{I}_{z}(q^{z}) $, 
\begin{equation*}
M_{ \mathrm{row} } ( q^{z} ) = M_{ \mathrm{row} } ( 1^{z'} ) =0.
\end{equation*}
This contradicts the setting that $ M_{ \mathrm{row} } ( q^{z} ) >0 $. Therefore, there exist $ \tilde{i}, \tilde{k} $ such that $q^{z}_{ \tilde{i} \tilde{k} } >0 $ but $1^{z'}_{  \tilde{i} \tilde{k} } =0$.

Then, we define the row label assignment $\tilde{z}$ as follows. Let
\begin{equation*}
\tilde{z}(i) =
\begin{cases}
z'(i), & \text{if $i \neq \tilde{i}$}  \\
\tilde{k}, & \text{if $i = \tilde{i}$}
\end{cases}.
\end{equation*}
Since $1^{z'}_{  \tilde{i} \tilde{k} } =0$, then $ z'( \tilde{i} ) \neq \tilde{k} = \tilde{z}( \tilde{i} ) $. We define $k' = z'( \tilde{i} )$, then $ k' \neq \tilde{k} $. Furthermore, let $ 1^{\tilde{z}} $ be the matrix defined as $ 1^{\tilde{z}}_{ik} = 1_{ \{  \tilde{z}(i) =k \}} $. It is easily verified that $1^{\tilde{z}} \in \mathcal{I}_{z}( q^{z} )$. Define for $k=1, \dots, K$,
\begin{align*}
  \mathcal{A}_{k}' = \{ 1 \leq i \leq m: z'(i) =k \}, \ \text{and} \ 
  \widetilde{\mathcal{A} }_{k} =  \{ 1 \leq i \leq m: \tilde{z} (i) =k \}.
\end{align*}
Noting that $M_{ \mathrm{row} } ( 1^{z'} ) =0  $ and applying Equation \eqref{eqid:anotherformofMrowcol} in the main text, there exists a permutation $s \in S_{K}$ such that
\begin{equation*}
 \frac{1}{m} \sum_{k=1}^K \left| \mathcal{A}_{s(k)}^{*} \cap \mathcal{A}_{k}' \right| = 1 = \frac{1}{m} \sum_{k=1}^K \left| \mathcal{A}_{s(k)}^{*}  \right|.
\end{equation*}
This implies that $\mathcal{A}_{s(k)}^{*} = \mathcal{A}_{k}'  $ for all $k = 1,\dots, K$. Without loss of generality, we may assume $s$ is the identity mapping, and thus, $ \mathcal{A}_{k}^{*} = \mathcal{A}_{k}' $. It follows by the definition of $1^{\tilde{z}}$ that
\begin{equation*}
\widetilde{\mathcal{A} }_{k} = 
\begin{cases}
\mathcal{A}_{ \tilde{k} }^{*} \cup \{ \tilde{i} \}, & \text{if $k = \tilde{k}$} \\
\mathcal{A}_{ k' }^{*} \setminus \{ \tilde{i} \}, & \text{if $k = k'$}  \\
\mathcal{A}_{ k }^{*}, & \text{if $k \neq \tilde{k}$ and $k \neq k'$ }
\end{cases}.
\end{equation*}
Note that $k' = z'( \tilde{i} ) \neq \tilde{k}$. By a direct calculation, we have
\begin{equation}
\left|  \mathcal{A}_{ k }^{*} \cap  \widetilde{\mathcal{A} }_{k'} \right| =
\begin{cases}
0, & \text{if $k \neq k'$} \\
|  \mathcal{A}_{ k' }^{*} | -1, & \text{if $k = k'$}   
\end{cases}.   \label{eqid:cardofA}
\end{equation}
Consider any permutation $s \in S_K$. By Condition \ref{cond:c1}, we have $ |  \mathcal{A}_{ s(k') }^{*} | >0$, and thus, $ |  \mathcal{A}_{ s(k') }^{*} | -1 \geq 0$. Then, it follows from \eqref{eqid:cardofA} that
\begin{equation*}
\left|  \mathcal{A}_{ s(k') }^{*} \cap  \widetilde{\mathcal{A} }_{k'} \right| \leq  \left|  \mathcal{A}_{ s(k') }^{*} \right| -1.
\end{equation*}
Together with \eqref{eqid:anotherformfoR}, we have
\begin{align*}
1 - \sum_{k=1}^{K} \mathbb{R}_{ s(k), k } \left( 1^{z^{*}}, 1^{ \tilde{z} } \right) &= 1 - \frac{1}{m} \sum_{k=1}^{K} \left|  \mathcal{A}_{ s(k) }^{*} \cap  \widetilde{\mathcal{A} }_{k} \right|  \\
& = 1- \frac{1}{m} \sum_{k \neq k'} \left|  \mathcal{A}_{ s(k) }^{*} \cap  \widetilde{\mathcal{A} }_{k} \right| - \frac{1}{m} \left|  \mathcal{A}_{ s(k') }^{*} \cap  \widetilde{\mathcal{A} }_{k'} \right| \\
& \geq 1- \frac{1}{m} \sum_{k \neq k'} \left|  \mathcal{A}_{ s(k) }^{*}  \right| - \frac{1}{m}\left(  \left|  \mathcal{A}_{ s(k') }^{*} \right| -1 \right) \\
& = 1 - \frac{1}{m} \sum_{k =1}^{K} \left|  \mathcal{A}_{ s(k) }^{*}  \right| + \frac{1}{m} \\
&= \frac{1}{m}.
\end{align*}
Since the permutation $s$ is arbitrary, we have
\begin{equation*}
M_{ \mathrm{row} } (1^{\tilde{z}}) \geq \frac{1}{m}.
\end{equation*} 
This contradicts with the assumption that $ M_{ \mathrm{row} } (1^{z}) =0 $ for all $ 1^{z} \in \mathcal{I}_{z} (q^z) $. Therefore, there exists $1^{\tilde{z}} \in \mathcal{I}_{z}(q^z)$ such that $M_{ \mathrm{row} } (1^{\tilde{z}}) > 0$.

Choose any $ 1^{\tilde{w}} \in \mathcal{I}_{w}(q^{w})$. It follows by (ii) that
\begin{equation*}
\bar{J}(1^{z^*}, 1^{w^*}, \theta^*, \lambda^*) - \bar{J}( 1^{ \tilde{z}} , 1^{\tilde{w}} , \theta, \lambda) > 0.
\end{equation*}
Define 
\begin{equation*}
v_{z} := \inf_{(i,k) : 1^{\tilde{z}}_{ik} =1} q^{z}_{ik}.
\end{equation*}
It follows by the definition of $\mathcal{I}_{z}(q^{z})$ and the fact $ 1^{\tilde{z}} \in \mathcal{I}_{z}(q^z) $ that $v_z >0$. Similarly, define
\begin{equation*}
v_{w} := \inf_{(j,l) : 1^{\tilde{w}}_{jl} =1} q^{w}_{jl}.
\end{equation*}
Then, $v_{w} >0$. It is obvious that $q^{z}_{ik} \geq v_{z} 1^{\tilde{z}}_{ik}$ for any $i$ and $k$, and similarly, $q^{w}_{jl} \geq v_{w} 1^{\tilde{w}}_{jl} $ for any $j$ and $l$. Using \eqref{eqm:formulafordifferenceJ} twice, we have
\begin{eqnarray*}
&& \bar{J} \left( 1^{z^*}, 1^{w^*}, \theta^*, \lambda^* \right) - \bar{J} \left( q^z, q^w, \theta, \lambda \right) \\
&=& \sum_{k=1}^K \sum_{l=1}^L \sum_{k'=1}^K \sum_{l'=1}^L \sum_{i=1}^m \sum_{j=1}^n 1_{ik}^{z^*} 1_{jl}^{w^*} q_{ik'}^z q_{jl'}^w \text{ KL}(\theta_{il}^* \lambda_{jk}^*, \theta_{il'} \lambda_{jk'}) \\
&\geq& v_{z} v_{w} \sum_{k=1}^K \sum_{l=1}^L \sum_{k'=1}^K \sum_{l'=1}^L \sum_{i=1}^m \sum_{j=1}^n 1_{ik}^{z^*} 1_{jl}^{w^*} 1_{ik'}^{\tilde{z}} 1_{jl'}^{\tilde{w}} \text{ KL}(\theta_{il}^* \lambda_{jk}^*, \theta_{il'} \lambda_{jk'}) \\
&=& v_{z} v_{w} \left( \bar{J} \left( 1^{z^*}, 1^{w^*}, \theta^*, \lambda^* \right) - \bar{J} \left( 1^{\tilde{z}},  1^{\tilde{w}} , \theta, \lambda \right) \right) \\
&>& 0.
\end{eqnarray*}
We conclude the proof.
\end{proof}

\subsection{Proof of Theorem \ref{thm:identification}}
\begin{proof}
The assertion in the theorem is equivalent to requiring that, for any $q^z \in \mathcal{C}_{z}$ and $q^w \in \mathcal{C}_{w}$ such that if $ M_{\mathrm{row}} (q^z)  > 0 $ or $M_{\mathrm{col}} (q^{w}) > 0$, then $ \bar{J} \left( 1^{z^*}, 1^{w^*}, \theta^*, \lambda^* \right) - \bar{J} \left( q^z, q^w, \theta, \lambda \right) >0 $. Furthermore, by applying Lemma \ref{lm:restrict}, this is equivalent to stating that for any $1^z \in \mathcal{I}_{z}$ and $1^w \in \mathcal{I}_{w}$ such that if $ M_{\mathrm{row}} (1^z)  > 0 $ or $M_{\mathrm{col}} (1^{w}) > 0$, then $ \bar{J} \left( 1^{z^*}, 1^{w^*}, \theta^*, \lambda^* \right) - \bar{J} \left( 1^z, 1^w, \theta, \lambda \right) >0 $. Therefore, it suffices to prove that
\begin{equation*}
\bar{J} \left( 1^{z^*}, 1^{w^*}, \theta^*, \lambda^* \right) - \bar{J} \left( 1^z, 1^w, \theta, \lambda \right) =0
\end{equation*}
only if $M_{\mathrm{row}} (1^z) =0  $ and $M_{\mathrm{col}} (1^{w}) = 0$.

Define
\begin{equation*}
\mathcal{A}^{*}_{k} = \{ 1 \leq i \leq m: z^{*}(i) =k \}, \quad \mathcal{A}_{k} = \{ 1 \leq i \leq m: z(i) =k \}
\end{equation*}
for $k = 1, \dots, K$, and
\begin{equation*}
\mathcal{B}^{*}_{l} = \{ 1 \leq j \leq n: z^{*}(j) =l \}, \quad \mathcal{B}_{l} = \{ 1 \leq j \leq n: z(j) =l \}
\end{equation*}
for $l = 1, \dots, L$. Applying \eqref{eqm:formulafordifferenceJ} and noting that $1^{z^{*}}_{ik} = 1_{ \{ z^{*}(i) =k \} }$, $1^{w^{*}}_{jl} = 1_{ \{  w^{*}(j) =l \} }$, $1^{z}_{ik'} = 1_{ \{ z(i) =  k' \} }$, and $1^{w}_{jl'} = 1_{ \{  w(j) = l' \} }$, it follows
\begin{eqnarray*}
&& \bar{J} \left( 1^{z^*}, 1^{w^*}, \theta^*, \lambda^* \right) - \bar{J} \left( 1^z, 1^w, \theta, \lambda \right) \\
&=& \sum_{k=1}^K \sum_{l=1}^L \sum_{k'=1}^K \sum_{l'=1}^L \sum_{i=1}^m \sum_{j=1}^n 1_{ik}^{z^*} 1_{jl}^{w^*} 1_{ik'}^z 1_{jl'}^w \text{ KL}(\theta_{il}^* \lambda_{jk}^*, \theta_{il'} \lambda_{jk'}) \\
&=& \sum_{i=1}^m \sum_{j=1}^n \mathrm{KL} \left( \theta_{i w^{*}(j)}^* \lambda_{j z^{*}(i)}^*, \theta_{i w(j)} \lambda_{j z(i)} \right).
\end{eqnarray*}
Therefore,
\begin{align*}
 & \bar{J} \left( 1^{z^*}, 1^{w^*}, \theta^*, \lambda^* \right) - \bar{J} \left( 1^z, 1^w, \theta, \lambda \right) =0 \\
 &\Longleftrightarrow \mathrm{KL} \left( \theta_{i w^{*}(j)}^* \lambda_{j z^{*}(i)}^*, \theta_{i w(j)} \lambda_{j z(i)} \right), \ \text{for all $ i =1, \dots, m$, and $ j = 1, \dots, n$} \\
 &\Longleftrightarrow  \theta_{i w^{*}(j)}^* \lambda_{j z^{*}(i)}^* = \theta_{i w(j)} \lambda_{j z(i)} , \ \text{for all $ i =1, \dots, m$, and $ j = 1, \dots, n$} \\
 &\Longleftrightarrow \text{for all $ k =1, \dots, K$, and $ l = 1, \dots, L$ such that $ \mathcal{A}_{k} \neq \varnothing $ and $ \mathcal{B}_{l} \neq \varnothing $}, \\
 & \qquad \left(  \theta_{i w^{*}(j)}^* \lambda_{j z^{*}(i)}^* \right)_{i \in \mathcal{A}_{k}, j \in \mathcal{B}_{l}} = \left( \theta_{i w(j)} \lambda_{j z(i)} \right)_{i \in \mathcal{A}_{k}, j \in \mathcal{B}_{l}}  \\
&\Longleftrightarrow \text{for all $ k =1, \dots, K$, and $ l = 1, \dots, L$ such that $ \mathcal{A}_{k} \neq \varnothing $ and $ \mathcal{B}_{l} \neq \varnothing $}, \\
& \qquad \left(  \theta_{i w^{*}(j)}^* \lambda_{j z^{*}(i)}^* \right)_{i \in \mathcal{A}_{k}, j \in \mathcal{B}_{l}}  = \left( \theta_{i l} \lambda_{j k} \right)_{i \in \mathcal{A}_{k}, j \in \mathcal{B}_{l}},
\end{align*} 
where the second equivalence follows from the fact that $\mathrm{KL}(a,b) =0$ if and only if $a=b$, as directly implied by Lemma 16 in \cite{zhao2024variational}. The last statement in the last display implies that $ (  \theta_{i w^{*}(j)}^* \lambda_{j z^{*}(i)}^* )_{i \in \mathcal{A}_{k}, j \in \mathcal{B}_{l}} $ is a rank-$1$ matrix for all $ k =1, \dots, K$, and $ l = 1, \dots, L$ such that $ \mathcal{A}_{k} \neq \varnothing $ and $ \mathcal{B}_{l} \neq \varnothing $.

In the next step, we use the result above, together with Conditions \ref{cond:c1} and \ref{cond:c2}, to demonstrate that 
\begin{equation}
\text{for any $k \in \{1, \dots, K\}$, $\mathcal{A}_{k} = \mathcal{A}^{*}_{k'}$ for some $k' \in \{1, \dots, K\}$.}  \label{eqid:apermutationargument}
\end{equation}
Since $\mathcal{A}_1, \dots, \mathcal{A}_{K}$ are pairwise disjoint, it follows that there exists a permutation $s \in S_K$ such that $\mathcal{A}_{k} = \mathcal{A}^{*}_{s(k)} $ for all $k= 1, \dots, K$.

Assume the statement in \eqref{eqid:apermutationargument} is not correct. We first claim that there exist $k$, $k_1'$, and $ k_2' \in \{1, \dots, K\}$ such that $k_1' \neq k_2'$, $ \mathcal{A}_{k} \cap \mathcal{A}^{*}_{k_1'} \neq \varnothing $, and $ \mathcal{A}_{k} \cap \mathcal{A}^{*}_{k_2'} \neq \varnothing $. We prove this claim as follows. For any $k$ satisfying $\mathcal{A}_k \neq \varnothing$, there must exist some $ t(k) \in \{1, \dots, K\}$ such that $\mathcal{A}_{k} \cap \mathcal{A}^{*}_{t(k)} \neq \varnothing$. Consider the set
\begin{equation}
\left\{ k: \mathcal{A}_{k} \nsubseteq \mathcal{A}^{*}_{t(k)} \text{ and } \mathcal{A}_{k} \neq \varnothing \right\}.  \label{eqid:asetofk}
\end{equation}
This set cannot be empty. Otherwise, $ \mathcal{A}_{k} \subseteq \mathcal{A}^{*}_{t(k)} $ for all $k$ such that $\mathcal{A}_{k} \neq \varnothing $. However, both $\mathcal{A}_{1}, \dots, \mathcal{A}_{K}$ and $\mathcal{A}_{1}^{*}, \dots, \mathcal{A}_{K}^{*}$ form partitions of $\{1, \dots, m\}$, and Condition \ref{cond:c1} ensures that $\mathcal{A}_{k}^{*} \neq \varnothing$ for all $k$. It follows that $ \mathcal{A}_k \neq \varnothing $ and $\mathcal{A}_{k} = \mathcal{A}^{*}_{t(k)}$ for all $k \in \{1, \dots, K\}$. This leads to a contradiction, since we assumed that the statement in \eqref{eqid:apermutationargument} is false. Therefore, the set in \eqref{eqid:asetofk} is not empty, and we choose an arbitrary $k$ from this set. It follows that $ \mathcal{A}_{k} \nsubseteq \mathcal{A}^{*}_{t(k)} $ and $\mathcal{A}_{k} \neq \varnothing$. Then, there exists $k_1' \in \{1, \dots, K\}$ such that $k_1' \neq t(k)$ and $ \mathcal{A}_{k} \cap \mathcal{A}^{*}_{k_1'} \neq \varnothing $. Furthermore, note that $\mathcal{A}_{k} \cap \mathcal{A}^{*}_{t(k)} \neq \varnothing$. We conclude the claim by letting $k_2' = t(k)$.

Let $k, k_1'$ and $k_2'$ be as specified in the preceding claim. Choose any $i_1 \in \mathcal{A}_{k} \cap \mathcal{A}^{*}_{k_1'}$ and $i_2 \in \mathcal{A}_{k} \cap \mathcal{A}^{*}_{k_2'}$. Then, $i_1 \neq i_2$. Furthermore, let $l$ be as specified in the second part of Condition \ref{cond:c2}; that is, \eqref{eqid:secondpartC2} holds for all distinct $ k_1, k_2$. 

Note that for any $ l' \in \{1, \dots, L\}$ such that $ \mathcal{B}_{l}^{*} \cap \mathcal{B}_{l'} \neq \varnothing $, the $2 \times | \mathcal{B}_{l}^{*} \cap \mathcal{B}_{l'} | $ matrix 
\begin{equation*}
\left( \theta_{i w^{*}(j)}^* \lambda_{j z^{*}(i)}^* \right)_{i \in \{i_1,  i_2\}, j \in \mathcal{B}_{l}^{*} \cap \mathcal{B}_{l'} }
\end{equation*}
is a submatrix of
\begin{equation*}
\left( \theta_{i w^{*}(j)}^* \lambda_{j z^{*}(i)}^* \right)_{i \in \mathcal{A}_k , j \in \mathcal{B}_{l'} },
\end{equation*}
and thus, is also a rank-$1$ matrix. Consequently, there exists $\nu_{l'} >0$ such that
\begin{equation*}
\left( \theta_{i w^{*}(j)}^* \lambda_{j z^{*}(i)}^* \right)_{i \in \{i_1\}, j \in \mathcal{B}_{l}^{*} \cap \mathcal{B}_{l'} }
= \nu_{l'} \left( \theta_{i w^{*}(j)}^* \lambda_{j z^{*}(i)}^* \right)_{i \in \{ i_2\}, j \in \mathcal{B}_{l}^{*} \cap \mathcal{B}_{l'} }.
\end{equation*}
Furthermore, noting that $i_1 \in  \mathcal{A}^{*}_{k_1'}$ and $i_2 \in \mathcal{A}^{*}_{k_2'}$, it follows that
\begin{equation*}
\left(  \theta_{i w^{*}(j)}^* \lambda_{j z^{*}(i)}^* \right)_{i \in \{i_1\}, j \in \mathcal{B}_{l}^{*} \cap \mathcal{B}_{l'} }
= \left(  \theta_{i_1 l}^* \lambda_{j k_1'}^* \right)_{i \in \{i_1\}, j \in \mathcal{B}_{l}^{*} \cap \mathcal{B}_{l'}}
= \theta_{i_1 l}^* \left( \lambda_{\mathcal{B}_{l}^{*} \cap \mathcal{B}_{l'}, \, k_1'}^* \right)^{\intercal}
\end{equation*}
and
\begin{equation*}
\left(  \theta_{i w^{*}(j)}^* \lambda_{j z^{*}(i)}^* \right)_{i \in \{i_2\}, j \in \mathcal{B}_{l}^{*} \cap \mathcal{B}_{l'} }
= \left(  \theta_{i_2 l}^* \lambda_{j k_2'}^* \right)_{i \in \{i_2\}, j \in \mathcal{B}_{l}^{*} \cap \mathcal{B}_{l'} }
= \theta_{i_2 l}^* \left( \lambda_{\mathcal{B}_{l}^{*} \cap \mathcal{B}_{l'} , \, k_2'}^* \right)^{\intercal}.
\end{equation*}
Combining the last three displays, it yields
\begin{equation*}
\theta_{i_1 l}^* \left( \lambda_{\mathcal{B}_{l}^{*} \cap \mathcal{B}_{l'} , \, k_1'}^* \right)^{\intercal}
= \nu_{l'} \theta_{i_2 l}^* \left( \lambda_{\mathcal{B}_{l}^{*}  \cap \mathcal{B}_{l'} , \, k_2'}^* \right)^{\intercal},
\end{equation*}
and thus,
\begin{equation}
\left\| \lambda_{\mathcal{B}_{l}^{*} \cap \mathcal{B}_{l'} , \, k_1'}^* - \frac{ \nu_{l'} \theta_{i_2 l}^* }{ \theta_{i_1 l}^*  }  \lambda_{\mathcal{B}_{l}^{*}  \cap \mathcal{B}_{l'} , \, k_2'}^* \right\|_{\mathrm{new}}^{2}=0.  \label{eqid:normnew0}
\end{equation}
We then estimate the range of $  \nu_{l'} \theta_{i_2 l}^* / \theta_{i_1 l}^* $. Since $ \mathcal{B}_{l}^{*} \cap \mathcal{B}_{l'} \neq \varnothing $, we choose any $ j' \in \mathcal{B}_{l}^{*} \cap \mathcal{B}_{l'} $. Then,
\begin{equation*}
\theta_{i_1 l}^* \lambda_{ j' k_{1}' }^{*}
= \nu_{l'} \theta_{i_2 l}^* \lambda_{ j' k_{2}' }^{*}.
\end{equation*}
Since $ \lambda_{ j' k_{1}' }^{*}, \lambda_{ j' k_{2}' }^{*} \in [ \eta_{\mathrm{min}}, \eta_{\mathrm{max}} ]$, we have
\begin{equation*}
 \frac{ \nu_{l'} \theta_{i_2 l}^* }{ \theta_{i_1 l}^*  } = \frac{ \lambda_{ j' k_{1}' }^{*} }{ \lambda_{ j' k_{2}' }^{*} } 
 \in \left[ \frac{\eta_{\mathrm{min}}}{\eta_{\mathrm{max}}},  \frac{\eta_{\mathrm{max}}}{\eta_{\mathrm{min}}} \right].
\end{equation*}
Together with \eqref{eqid:normnew0}, it follows that
\begin{equation*}
\inf_{ c_{l'} \in \left[ \frac{\eta_{\mathrm{min}}}{\eta_{\mathrm{max}}},  \frac{\eta_{\mathrm{max}}}{\eta_{\mathrm{min}}} \right] } \,  \left\| \lambda_{\mathcal{B}_{l}^{*} \cap \mathcal{B}_{l'} , \, k_1'}^* - c_{l'} \lambda_{\mathcal{B}_{l}^{*}  \cap \mathcal{B}_{l'} , \, k_2'}^* \right\|_{\mathrm{new}}^{2}=0.
\end{equation*}
When $ \mathcal{B}_{l}^{*} \cap \mathcal{B}_{l'} = \varnothing $, the preceding equality holds by convention. Letting $l'$ vary from $1$ to $L$, we have
\begin{equation*}
\inf_{ c_1, \dots, c_L \in \left[ \frac{\eta_{\mathrm{min}}}{\eta_{\mathrm{max}}},  \frac{\eta_{\mathrm{max}}}{\eta_{\mathrm{min}}} \right] } \  \sum_{l' =1}^{L} \left\| \lambda_{\mathcal{B}_{l}^{*} \cap \mathcal{B}_{l'} , \, k_1'}^* - c_{l'} \lambda_{\mathcal{B}_{l}^{*}  \cap \mathcal{B}_{l'} , \, k_2'}^* \right\|_{\mathrm{new}}^{2}=0.
\end{equation*}
In other words, we construct a partition of $ \mathcal{B}_{l}^{*} $ into subsets $ \mathcal{B}_{l}^{*} \cap \mathcal{B}_{1}, \dots, \mathcal{B}_{l}^{*} \cap \mathcal{B}_{L} $ such that
\begin{equation*}
\inf_{ c_1, \dots, c_L \in \left[ \frac{\eta_{\mathrm{min}}}{\eta_{\mathrm{max}}},  \frac{\eta_{\mathrm{max}}}{\eta_{\mathrm{min}}} \right] } \  \sum_{l' =1}^{L} \left\| \lambda_{\mathcal{B}_{l}^{*} \cap \mathcal{B}_{l'} , \, k_1'}^* - c_{l'} \lambda_{\mathcal{B}_{l}^{*}  \cap \mathcal{B}_{l'} , \, k_2'}^* \right\|_{\mathrm{new}}^{2}=0.
\end{equation*}
This contradicts with \eqref{eqid:secondpartC2} in Condition \ref{cond:c2}. Therefore, our assumption is false and the statement in \eqref{eqid:apermutationargument} holds. Then, there exists a permutation $s \in S_K$ such that $\mathcal{A}_{k} = \mathcal{A}_{s(k)}^{*}$ for all $k \in \{ 1, \dots, K \}$. Applying \eqref{eqid:anotherformfoR}, we have
\begin{equation*}
1 - \sum_{k=1}^{K} \mathbb{R}_{s(k) k}( 1^{z^*}, 1^{z} ) = 1 - \frac{1}{m} \sum_{k=1}^{K} \left| \mathcal{A}_{s(k)}^{*} \cap  \mathcal{A}_{k} \right|
= 1 - \frac{1}{m} \sum_{k=1}^{K} \left| \mathcal{A}_{s(k)}^{*} \right| = 0. 
\end{equation*}
Therefore,
\begin{equation*}
M_{\mathrm{row}} ( 1^{z} ) = 0.
\end{equation*}
Similarly, $ M_{\mathrm{col}} ( 1^{w} ) = 0 $. We conclude the proof.
\end{proof}

\subsection{Proof of Theorem \ref{thm:wellseparatedness}}
\begin{proof}
By symmetry, we only prove the first inequality in the theorem.

For any $j=1, \dots, n$, since $\sum_{l=1}^{L} q_{jl}^{w} =1 $, there exists $l(j)$ (depending on $j$) such that $q_{j l(j)}^{w} \geq 1/L$. Let $\tilde{w}$ be a column label assignment defined by $\tilde{w}(j) = l(j)$ for each $j \in {1, \dots, n}$, and denote by $1^{\tilde{w}}$ the corresponding matrix in $\mathcal{I}_{w}$. Again, we define
\begin{equation*}
\widetilde{ \mathcal{B} }_{l} = \left\{ 1 \leq j \leq n: \tilde{w}(j) =l  \right\},
\end{equation*}
for $l = 1, \dots, L$. It is straightforward that, for all $j = 1, \dots, n$, 
\begin{equation*}
q_{j, \tilde{w}(j)}^{w} \geq \frac{1}{L}.
\end{equation*}
Therefore,
\begin{eqnarray*}
&& \bar{J}(1^{z^*}, 1^{w^*}, \theta^*, \lambda^*) - \bar{J}(q^z, q^w, \theta, \lambda) \\
&=& \sum_{k=1}^K \sum_{l=1}^L \sum_{k'=1}^K \sum_{l'=1}^L \sum_{i=1}^m \sum_{j=1}^n 1_{ik}^{z^*} 1_{jl}^{w^*} q_{ik'}^z q_{jl'}^w \mathrm{ KL}(\theta_{il}^* \lambda_{jk}^*, \theta_{il'} \lambda_{jk'})  \\
&\geq& \sum_{k'=1}^K \sum_{i=1}^m \sum_{j=1}^n 1_{i z^{*}(i) }^{z^*} 1_{j w^{*}(j) }^{w^*} q_{ik'}^z q_{j \tilde{w}(j) }^w \mathrm{ KL} \left(\theta_{i w^{*}(j)}^* \lambda_{j z^{*}(i)}^*, \theta_{i\tilde{w}(j)} \lambda_{jk'}\right) \\
&\geq& \frac{1}{L} \sum_{k'=1}^K \sum_{i=1}^m \sum_{j=1}^n q_{ik'}^z  \mathrm{ KL} \left(\theta_{i w^{*}(j)}^* \lambda_{j z^{*}(i)}^*, \theta_{i\tilde{w}(j)} \lambda_{jk'}\right) \\
&=& \frac{1}{L} \sum_{k'=1}^K \left( \sum_{k=1}^K \sum_{i \in \mathcal{A}_{k}^{*}} \right) \left( \sum_{l=1}^L \sum_{j \in \mathcal{B}_{l}^{*}} \right) q_{ik'}^z  \mathrm{ KL} \left(\theta_{i w^{*}(j)}^* \lambda_{j z^{*}(i)}^*, \theta_{i\tilde{w}(j)} \lambda_{jk'}\right) \\
&=& \frac{1}{L} \sum_{k'=1}^K   \sum_{k=1}^K  \sum_{l=1}^L  \sum_{i \in \mathcal{A}_{k}^{*}}   \sum_{j \in \mathcal{B}_{l}^{*}}   q_{ik'}^z  \mathrm{ KL} \left(\theta_{i w^{*}(j)}^* \lambda_{j z^{*}(i)}^*, \theta_{i\tilde{w}(j)} \lambda_{jk'}\right) \\
&=& \frac{1}{L} \sum_{k'=1}^K   \sum_{k=1}^K  \sum_{l=1}^L  \sum_{i \in \mathcal{A}_{k}^{*}}   \sum_{j \in \mathcal{B}_{l}^{*}}   q_{ik'}^z  \mathrm{ KL} \left(\theta_{i l}^* \lambda_{j k}^*, \theta_{i\tilde{w}(j)} \lambda_{jk'}\right).
\end{eqnarray*}
Noting that $\mathrm{KL}(ca, cb) = c \mathrm{KL}(a, b)$ for any $a,b,c >0$, we have
\begin{eqnarray*}
&& \bar{J}(1^{z^*}, 1^{w^*}, \theta^*, \lambda^*) - \bar{J}(q^z, q^w, \theta, \lambda) \\
&=& \frac{1}{L}  \sum_{l=1}^L  \sum_{k'=1}^K   \sum_{k=1}^K   \sum_{i \in \mathcal{A}_{k}^{*}}   \sum_{j \in \mathcal{B}_{l}^{*}}   q_{ik'}^z   \mathrm{ KL} \left(\theta_{i l}^* \lambda_{j k}^*, \theta_{i\tilde{w}(j)} \lambda_{jk'}\right) \\
&=& \frac{1}{L}  \sum_{l=1}^L  \sum_{k'=1}^K   \sum_{k=1}^K   \sum_{i \in \mathcal{A}_{k}^{*}}   \sum_{j \in \mathcal{B}_{l}^{*}}   q_{ik'}^z  \theta_{i l}^*  \, \mathrm{ KL} \left( \lambda_{j k}^*, \frac{ \theta_{i\tilde{w}(j)} }{\theta_{i l}^*}  \lambda_{jk'}\right) \\
&\geq& \frac{ \eta_{\mathrm{min}} }{L} \sum_{l=1}^L  \sum_{k'=1}^K   \sum_{k=1}^K   \sum_{i \in \mathcal{A}_{k}^{*}}   \sum_{j \in \mathcal{B}_{l}^{*}}   q_{ik'}^z   \mathrm{ KL} \left( \lambda_{j k}^*, \frac{ \theta_{i\tilde{w}(j)} }{\theta_{i l}^*}  \lambda_{jk'}\right) \\
&=& \frac{ \eta_{\mathrm{min}} }{L} \sum_{l=1}^L  \sum_{k'=1}^K   \sum_{k=1}^K   \sum_{i \in \mathcal{A}_{k}^{*}}  \sum_{l'=1}^{L}  \sum_{j \in \mathcal{B}_{l}^{*} \cap \widetilde{\mathcal{B}}_{l'} }   q_{ik'}^z   \mathrm{ KL} \left( \lambda_{j k}^*, \frac{ \theta_{i\tilde{w}(j)} }{\theta_{i l}^*}  \lambda_{jk'}\right) \\
&=& \frac{ \eta_{\mathrm{min}} }{L} \sum_{l=1}^L  \sum_{k'=1}^K   \sum_{k=1}^K   \sum_{i \in \mathcal{A}_{k}^{*}}  \sum_{l'=1}^{L}  \sum_{j \in \mathcal{B}_{l}^{*} \cap \widetilde{\mathcal{B}}_{l'} }   q_{ik'}^z   \mathrm{ KL} \left( \lambda_{j k}^*, \frac{ \theta_{i l' } }{\theta_{i l}^*}  \lambda_{jk'}\right).
\end{eqnarray*}
Applying Lemma 16 in \cite{zhao2024variational}, it follows that
\begin{align*}
 &\quad \mathrm{ KL} \left( \lambda_{j k}^*, \frac{ \theta_{i l' } }{\theta_{i l}^*}  \lambda_{jk'}\right) \\
&\geq \min\left\{ \left( \lambda_{j k}^* - \frac{ \theta_{i l' } }{\theta_{i l}^*}  \lambda_{jk'} \right)^2 \bigg/ \left( 6 \frac{ \theta_{i l' } }{\theta_{i l}^*}  \lambda_{jk'} \right),   \left|  \lambda_{j k}^* - \frac{ \theta_{i l' } }{\theta_{i l}^*}  \lambda_{jk'} \right| \right\}  \\
& \geq \min\left\{ \left( \lambda_{j k}^* - \frac{ \theta_{i l' } }{\theta_{i l}^*}  \lambda_{jk'} \right)^2 \bigg/ \left( 6 \frac{\eta_{\mathrm{max}}^2}{\eta_{\mathrm{min}}} \right),   \left|  \lambda_{j k}^* - \frac{ \theta_{i l' } }{\theta_{i l}^*}  \lambda_{jk'} \right| \right\}  \\
& \geq \min\left\{ \left( \lambda_{j k}^* - \frac{ \theta_{i l' } }{\theta_{i l}^*}  \lambda_{jk'} \right)^2 \bigg/ \left( 6 \mu \right),   \left|  \lambda_{j k}^* - \frac{ \theta_{i l' } }{\theta_{i l}^*}  \lambda_{jk'} \right| \right\} \\
& \geq \min\left\{ \left( \lambda_{j k}^* - \frac{ \theta_{i l' } }{\theta_{i l}^*}  \lambda_{jk'} \right)^2 \bigg/ \left( 6 \mu \right),   6 \mu \right\}  \\
&= \frac{1}{6 \mu} \left|  \lambda_{j k}^* - \frac{ \theta_{i l' } }{\theta_{i l}^*}  \lambda_{jk'} \right|_{\mathrm{new}}^{2}.  
\end{align*}
Combining the last two displays, we have
\begin{align}
& \bar{J}(1^{z^*}, 1^{w^*}, \theta^*, \lambda^*) - \bar{J}(q^z, q^w, \theta, \lambda)  \nonumber \\
&\geq \frac{ \eta_{\mathrm{min}} }{ 6 \mu L }  \sum_{l=1}^L  \sum_{k'=1}^K   \sum_{k=1}^K   \sum_{i \in \mathcal{A}_{k}^{*}}  \sum_{l'=1}^{L}  \sum_{j \in \mathcal{B}_{l}^{*} \cap \widetilde{\mathcal{B}}_{l'} }   q_{ik'}^z   \left|  \lambda_{j k}^* - \frac{ \theta_{i l' } }{\theta_{i l}^*}  \lambda_{jk'} \right|_{\mathrm{new}}^{2}  \nonumber \\
&= \frac{ \eta_{\mathrm{min}} }{ 6 \mu L }  \sum_{l=1}^L  \sum_{k'=1}^K   \sum_{k=1}^K   \sum_{i \in \mathcal{A}_{k}^{*}}  \sum_{l'=1}^{L}  q_{ik'}^z   \left\|  \lambda_{\mathcal{B}_{l}^{*} \cap \widetilde{\mathcal{B}}_{l'}, \, k}^* - \frac{ \theta_{i l' } }{\theta_{i l}^*}  \lambda_{\mathcal{B}_{l}^{*} \cap \widetilde{\mathcal{B}}_{l'}, \, k'} \right\|_{\mathrm{new}}^{2}  \nonumber \\
&\geq \frac{ \eta_{\mathrm{min}} }{ 6 \mu L }  \sum_{l=1}^L  \sum_{k'=1}^K   \sum_{k=1}^K   \sum_{i \in \mathcal{A}_{k}^{*}}  \sum_{l'=1}^{L}  q_{ik'}^z  \inf_{ c \in   \left[ \frac{ \eta_{\mathrm{min}} }{\eta_{\mathrm{max}}} ,  \frac{ \eta_{\mathrm{max}} }{\eta_{\mathrm{min}}}\right]}   \left\|  \lambda_{\mathcal{B}_{l}^{*} \cap \widetilde{\mathcal{B}}_{l'}, \, k}^* - c  \lambda_{\mathcal{B}_{l}^{*} \cap \widetilde{\mathcal{B}}_{l'}, \, k'} \right\|_{\mathrm{new}}^{2}  \nonumber \\
&= \frac{ \eta_{\mathrm{min}} }{ 6 \mu L }  \sum_{l=1}^L  \sum_{k'=1}^K   \sum_{k=1}^K   \sum_{i \in \mathcal{A}_{k}^{*}}   q_{ik'}^z \notag \\
&\qquad \inf_{ c_1, \dots, c_L \in   \left[ \frac{ \eta_{\mathrm{min}} }{\eta_{\mathrm{max}}} ,  \frac{ \eta_{\mathrm{max}} }{\eta_{\mathrm{min}}}\right]} \  \sum_{l'=1}^{L}  \left\|  \lambda_{\mathcal{B}_{l}^{*} \cap \widetilde{\mathcal{B}}_{l'}, \, k}^* - c_{l'}  \lambda_{\mathcal{B}_{l}^{*} \cap \widetilde{\mathcal{B}}_{l'}, \, k'} \right\|_{\mathrm{new}}^{2}.  \nonumber  \\
& \quad \quad\quad\quad\quad  \label{eqws:byproduct}
\end{align}

We then claim that for every $k' \in \{1, \dots, K\}$, there exists at most one $k \in \{1, \dots, K\}$ such that
\begin{equation}
\inf_{ c_1, \dots, c_L \in   \left[ \frac{ \eta_{\mathrm{min}} }{\eta_{\mathrm{max}}} ,  \frac{ \eta_{\mathrm{max}} }{\eta_{\mathrm{min}}}\right]} \  \sum_{l'=1}^{L}  \left\|  \lambda_{\mathcal{B}_{l}^{*} \cap \widetilde{\mathcal{B}}_{l'}, \, k}^* - c_{l'}  \lambda_{\mathcal{B}_{l}^{*} \cap \widetilde{\mathcal{B}}_{l'}, \, k'} \right\|_{\mathrm{new}}^{2}
< \frac{ \eta_{\mathrm{min}}^{4} }{ 2 ( \eta_{\mathrm{min}}^{4} + \eta_{\mathrm{max}}^{4} ) } \tau |\mathcal{B}_{l}^{*}|. \label{eqws:acontrolinequality}
\end{equation}
Otherwise, suppose there exist two $k_1, k_2 \in \{1, \dots, K\}$ satisfying the above inequality. Then, there exist $c_1, \dots, c_{L}, d_{1}, \dots, d_{L} \in [ \eta_{\mathrm{min}}/ \eta_{\mathrm{max}}, \eta_{\mathrm{max}}/\eta_{\mathrm{min}} ] $ such that
\begin{align*}
 & \sum_{l'=1}^{L} \left\|  \lambda_{\mathcal{B}_{l}^{*} \cap \widetilde{\mathcal{B}}_{l'}, \, k_1}^* - c_{l'}  \lambda_{\mathcal{B}_{l}^{*} \cap \widetilde{\mathcal{B}}_{l'}, \, k'} \right\|_{\mathrm{new}}^{2} < \frac{ \eta_{\mathrm{min}}^{4} }{ 2 ( \eta_{\mathrm{min}}^{4} + \eta_{\mathrm{max}}^{4} ) } \tau |\mathcal{B}_{l}^{*}|,
\end{align*}
and,
\begin{align*}
 & \sum_{l'=1}^{L} \left\|  \lambda_{\mathcal{B}_{l}^{*} \cap \widetilde{\mathcal{B}}_{l'}, \, k_2}^* - d_{l'}  \lambda_{\mathcal{B}_{l}^{*} \cap \widetilde{\mathcal{B}}_{l'}, \, k'} \right\|_{\mathrm{new}}^{2} < \frac{ \eta_{\mathrm{min}}^{4} }{ 2 ( \eta_{\mathrm{min}}^{4} + \eta_{\mathrm{max}}^{4} ) } \tau |\mathcal{B}_{l}^{*}|.
\end{align*}
Noting that $c_{l'}/d_{l'} \leq   \eta_{\mathrm{max}}^{2} /  \eta_{\mathrm{min}}^{2}$, it follows by the triangle inequality that
\begin{align*}
&\quad \sum_{l'=1}^{L} \left\|  \lambda_{\mathcal{B}_{l}^{*} \cap \widetilde{\mathcal{B}}_{l'}, \, k_1}^* - \frac{c_{l'}}{d_{l'}}    \lambda_{\mathcal{B}_{l}^{*} \cap \widetilde{\mathcal{B}}_{l'}, \, k_2}^* \right\|_{\mathrm{new}}^{2} \\
&= \sum_{l'=1}^{L} \Bigg\|  \left(\lambda_{\mathcal{B}_{l}^{*} \cap \widetilde{\mathcal{B}}_{l'}, \, k_1}^*  - c_{l'} \lambda_{\mathcal{B}_{l}^{*} \cap \widetilde{\mathcal{B}}_{l'}, \, k'}\right) -  \frac{c_{l'}}{d_{l'}}   \left( \lambda_{\mathcal{B}_{l}^{*} \cap \widetilde{\mathcal{B}}_{l'}, \, k_2}^* - d_{l'}  \lambda_{\mathcal{B}_{l}^{*} \cap \widetilde{\mathcal{B}}_{l'}, \, k'}\right) \Bigg\|_{\mathrm{new}}^{2} \\
&\leq \sum_{l'=1}^{L}  \Bigg(  \left\| \lambda_{\mathcal{B}_{l}^{*} \cap \widetilde{\mathcal{B}}_{l'}, \, k_1}^*  - c_{l'} \lambda_{\mathcal{B}_{l}^{*} \cap \widetilde{\mathcal{B}}_{l'}, \, k'} \right\|_{\mathrm{new}} \\
&\qquad + \left\| \frac{c_{l'}}{d_{l'}} \left(\lambda_{\mathcal{B}_{l}^{*} \cap \widetilde{\mathcal{B}}_{l'}, \, k_2}^* - d_{l'}  \lambda_{\mathcal{B}_{l}^{*} \cap \widetilde{\mathcal{B}}_{l'}, \, k'}\right) \right\|_{\mathrm{new}} \Bigg)^2 \\
&\leq 2 \sum_{l'=1}^{L}  \left\| \lambda_{\mathcal{B}_{l}^{*} \cap \widetilde{\mathcal{B}}_{l'}, \, k_1}^*  - c_{l'} \lambda_{\mathcal{B}_{l}^{*} \cap \widetilde{\mathcal{B}}_{l'}, \, k'} \right\|_{\mathrm{new}}^{2} \\ 
&\qquad + 2 \sum_{l'=1}^{L} \left\| \frac{c_{l'}}{d_{l'}} \left(\lambda_{\mathcal{B}_{l}^{*} \cap \widetilde{\mathcal{B}}_{l'}, \, k_2}^* - d_{l'}  \lambda_{\mathcal{B}_{l}^{*} \cap \widetilde{\mathcal{B}}_{l'}, \, k'}\right) \right\|^{2}_{\mathrm{new}} \\
&\leq 2 \sum_{l'=1}^{L}  \left\| \lambda_{\mathcal{B}_{l}^{*} \cap \widetilde{\mathcal{B}}_{l'}, \, k_1}^*  - c_{l'} \lambda_{\mathcal{B}_{l}^{*} \cap \widetilde{\mathcal{B}}_{l'}, \, k'} \right\|_{\mathrm{new}}^{2} \\
&\qquad + 2 \sum_{l'=1}^{L} \left[ \left(  \frac{c_{l'}}{d_{l'}} \right) \vee 1 \right]^{2} \left\|  \left(\lambda_{\mathcal{B}_{l}^{*} \cap \widetilde{\mathcal{B}}_{l'}, \, k_2}^* - d_{l'}  \lambda_{\mathcal{B}_{l}^{*} \cap \widetilde{\mathcal{B}}_{l'}, \, k'}\right) \right\|_{\mathrm{new}}^{2} \\
&\leq 2 \sum_{l'=1}^{L}  \left\| \lambda_{\mathcal{B}_{l}^{*} \cap \widetilde{\mathcal{B}}_{l'}, \, k_1}^*  - c_{l'} \lambda_{\mathcal{B}_{l}^{*} \cap \widetilde{\mathcal{B}}_{l'}, \, k'} \right\|_{\mathrm{new}}^{2} \\
&\qquad+ 2 \left( \frac{  \eta_{\mathrm{max}}^{2} }{  \eta_{\mathrm{min}}^{2} } \right)^{2} \sum_{l'=1}^{L} \left\|  \left(\lambda_{\mathcal{B}_{l}^{*} \cap \widetilde{\mathcal{B}}_{l'}, \, k_2}^* - d_{l'}  \lambda_{\mathcal{B}_{l}^{*} \cap \widetilde{\mathcal{B}}_{l'}, \, k'}\right) \right\|_{\mathrm{new}}^{2} \\
&< 2 \frac{ \eta_{\mathrm{min}}^{4} }{ 2 ( \eta_{\mathrm{min}}^{4} + \eta_{\mathrm{max}}^{4} ) } \tau |\mathcal{B}_{l}^{*}| + 2  \left( \frac{  \eta_{\mathrm{max}}^{2} }{  \eta_{\mathrm{min}}^{2} } \right)^{2}  \frac{ \eta_{\mathrm{min}}^{4} }{ 2 ( \eta_{\mathrm{min}}^{4} + \eta_{\mathrm{max}}^{4} ) } \tau |\mathcal{B}_{l}^{*}| \\
&=  \tau |\mathcal{B}_{l}^{*}|,
\end{align*}
where the last third inequality follows from the property of the newly defined metric, as presented in Section \ref{sec:theory}. Noting $\frac{c_{l'}}{d_{l'}} \in [\eta_{\mathrm{min}}^{2} /  \eta_{\mathrm{max}}^{2},  \eta_{\mathrm{max}}^{2} /  \eta_{\mathrm{min}}^{2}],$ the above result contradicts \eqref{eqws:secondpartC3} in Condition \ref{cond:c3}. Therefore, for every $k' \in \{1, \dots, K\}$, there exists at most one $k \in \{1, \dots, K\}$ such that \eqref{eqws:acontrolinequality} holds.

In the next step, we use the above result to construct a permutation $s_{l}$ of $\{1, \dots, K\}$ for each $l$. Here, the subscript $l$ means that the permutation $s_l$ depends on $l$. Let $\mathcal{K}_{l}$ be the set of $k'$ satisfying that there exists one and only one $k \in \{1, \dots, K\}$ such that \eqref{eqws:acontrolinequality} holds. Then, $\mathcal{K}_{l}$ is a subset of $\{1, \dots, K\}$. For any $k' \in \mathcal{K}_l$, let $s_{l}(k')$ be the $k$ such that \eqref{eqws:acontrolinequality} holds. Furthermore, our preceding result implies that $s_l$ is injective on $\mathcal{K}_l$. As a result, we could extend $s_l$ to a bijective mapping from $\{1, \dots, K\}$ to $\{1, \dots, K\}$, and thus, $s_l$ is a permutation of $\{1, \dots, K\}$.

In addition, we claim that for any $k \neq s_{l}(k')$, we have
\begin{equation}
\inf_{ c_1, \dots, c_L \in   \left[ \frac{ \eta_{\mathrm{min}} }{\eta_{\mathrm{max}}} ,  \frac{ \eta_{\mathrm{max}} }{\eta_{\mathrm{min}}}\right]} \  \sum_{l'=1}^{L}  \left\|  \lambda_{\mathcal{B}_{l}^{*} \cap \widetilde{\mathcal{B}}_{l'}, \, k}^* - c_{l'}  \lambda_{\mathcal{B}_{l}^{*} \cap \widetilde{\mathcal{B}}_{l'}, \, k'} \right\|_{\mathrm{new}}^{2}
\geq \frac{ \eta_{\mathrm{min}}^{4} }{ 2 ( \eta_{\mathrm{min}}^{4} + \eta_{\mathrm{max}}^{4} ) } \tau |\mathcal{B}_{l}^{*}|. \label{eqws:acontrolinequalityanotherpart}
\end{equation}
Indeed, if $k' \in \mathcal{K}_l$, then \eqref{eqws:acontrolinequalityanotherpart} holds for any $k \neq s_{l}(k')$ by the definition of $\mathcal{K}_l$; if $k' \notin \mathcal{K}_l $, the \eqref{eqws:acontrolinequalityanotherpart} holds for all $k \in \{1, \dots, K\}$.

With $s_l$ in hand, we proceed to evaluate the right-hand side of \eqref{eqws:byproduct}. Together with the last display, we have
\begin{eqnarray*}
&& \bar{J}(1^{z^*}, 1^{w^*}, \theta^*, \lambda^*) - \bar{J}(q^z, q^w, \theta, \lambda)  \\
&\geq& \frac{ \eta_{\mathrm{min}} }{ 6 \mu L }  \sum_{l=1}^L  \sum_{k'=1}^K   \sum_{k=1}^K   \sum_{i \in \mathcal{A}_{k}^{*}}   q_{ik'}^z  \\   
&&\qquad\inf_{ c_1, \dots, c_L \in   \left[ \frac{ \eta_{\mathrm{min}} }{\eta_{\mathrm{max}}} ,  \frac{ \eta_{\mathrm{max}} }{\eta_{\mathrm{min}}}\right]} \  \sum_{l'=1}^{L}  \left\|  \lambda_{\mathcal{B}_{l}^{*} \cap \widetilde{\mathcal{B}}_{l'}, \, k}^* - c_{l'}  \lambda_{\mathcal{B}_{l}^{*} \cap \widetilde{\mathcal{B}}_{l'}, \, k'} \right\|_{\mathrm{new}}^{2}  \\
&\geq& \frac{ \eta_{\mathrm{min}} }{ 6 \mu L }  \sum_{l=1}^L  \sum_{k'=1}^K   \sum_{k \neq s_{l}(k')}   \left(\sum_{i \in \mathcal{A}_{k}^{*}}   q_{ik'}^z \right) \  \\
&&\qquad \inf_{ c_1, \dots, c_L \in   \left[ \frac{ \eta_{\mathrm{min}} }{\eta_{\mathrm{max}}} ,  \frac{ \eta_{\mathrm{max}} }{\eta_{\mathrm{min}}}\right]} \  \sum_{l'=1}^{L}  \left\|  \lambda_{\mathcal{B}_{l}^{*} \cap \widetilde{\mathcal{B}}_{l'}, \, k}^* - c_{l'}  \lambda_{\mathcal{B}_{l}^{*} \cap \widetilde{\mathcal{B}}_{l'}, \, k'} \right\|_{\mathrm{new}}^{2}  \\
&\geq& \frac{ \eta_{\mathrm{min}} }{ 6 \mu L }  \sum_{l=1}^L  \sum_{k'=1}^K   \sum_{k \neq s_{l}(k')}   \left(\sum_{i \in \mathcal{A}_{k}^{*}}   q_{ik'}^z \right) \frac{ \eta_{\mathrm{min}}^{4} }{ 2 ( \eta_{\mathrm{min}}^{4} + \eta_{\mathrm{max}}^{4} ) } \tau |\mathcal{B}_{l}^{*}| \\
&=& \frac{ \tau }{ 12 \mu L }  \frac{ \eta_{\mathrm{min}}^{5} }{  \eta_{\mathrm{min}}^{4} + \eta_{\mathrm{max}}^{4}  } \sum_{l=1}^L  \sum_{k'=1}^K   \sum_{k \neq s_{l}(k')}   \left(\sum_{i \in \mathcal{A}_{k}^{*}}   q_{ik'}^z \right) |\mathcal{B}_{l}^{*}|.
\end{eqnarray*}
Note that
\begin{eqnarray*}
&& \sum_{l=1}^L  \sum_{k'=1}^K   \sum_{k \neq s_{l}(k')}   \left(\sum_{i \in \mathcal{A}_{k}^{*}}   q_{ik'}^z \right) |\mathcal{B}_{l}^{*}| \\ 
&=& \sum_{l=1}^L |\mathcal{B}_{l}^{*}|  \sum_{k'=1}^K   \sum_{k \neq s_{l}(k')}  \sum_{i \in \mathcal{A}_{k}^{*}}   q_{ik'}^z \\
&=& \sum_{l=1}^L |\mathcal{B}_{l}^{*}|  \left(  \sum_{k'=1}^K   \sum_{k =1}^{K}  \sum_{i \in \mathcal{A}_{k}^{*}}   q_{ik'}^z  - \sum_{k'=1}^{K}  \sum_{i \in \mathcal{A}_{s_l(k')}^*}   q_{ik'}^z \right) \\
&=& \sum_{l=1}^L |\mathcal{B}_{l}^{*}|  \left(  \sum_{k'=1}^K   \sum_{i =1}^{m}   q_{ik'}^z  - \sum_{k'=1}^{K}  \sum_{i =1}^{m} 1_{i s_{l}(k')}^{z^{*}}   q_{ik'}^z \right)  \\
&=& \sum_{l=1}^L |\mathcal{B}_{l}^{*}|  \left( \sum_{i =1}^{m} \sum_{k'=1}^K      q_{ik'}^z  - \sum_{k'=1}^{K}  m \mathbb{R}_{s_{l}(k'), k'} ( 1^{z^{*}}, q^{z} ) \right) \\
&=& \sum_{l=1}^L |\mathcal{B}_{l}^{*}|  \left( m  -  m \sum_{k'=1}^{K}   \mathbb{R}_{s_{l}(k'), k'} ( 1^{z^{*}}, q^{z} ) \right) \\
&\geq& m \sum_{l=1}^L |\mathcal{B}_{l}^{*}| M_{\mathrm{row}} (q^{z}) \\
&=& mn M_{\mathrm{row}} (q^{z}) ,
\end{eqnarray*}
where the last second equality is due to the fact that $\sum_{k} q_{ik}^{z} =1$ and the last inequality follows by the definition of $M_{\mathrm{row}}(q^z)$. Combining the last two displays, we have
\begin{equation*}
\bar{J}(1^{z^*}, 1^{w^*}, \theta^*, \lambda^*) - \bar{J}(q^z, q^w, \theta, \lambda)
\geq  \frac{ \tau }{ 12 \mu L }  \frac{ \eta_{\mathrm{min}}^{5} }{  \eta_{\mathrm{min}}^{4} + \eta_{\mathrm{max}}^{4}  }  mn M_{\mathrm{row}} (q^{z}) .
\end{equation*}
We conclude the proof.
\end{proof}

\subsection{Proof of Theorem \ref{thm:labelconsistency}}
\begin{proof}
Let $\Phi^*=({\pi}^*,{\rho}^*,\theta^*,\lambda^*)$. By Theorems \ref{thm:concentration} and \ref{thm:wellseparatedness}, and noting that $C_1 \asymp \tilde{\eta}^2$, we have that if $mn \tilde{\eta}^2 / ((m+n) (\log \tilde{\eta})^2) \rightarrow \infty$,
\begin{align*}
 & \,\, \mathbb{P} \left (M_\textnormal{row}(\hat{q}^{{z}})\geq \epsilon \big \vert z^*, w^* \right ) \\
 \leq  & \,\, \mathbb{P} \left ( \bar{J}({1}^{{z}^*},{1}^{{w}^*},\theta^*,\lambda^*)-\bar{J}(\hat{q}^{{z}},\hat{q}^{{w}}, \hat{\theta},\hat{\lambda}) \geq  C_1 mn   \epsilon  \big \vert z^*,w^*   \right ) \\
 = & \,\, \mathbb{P}\left (\bar{J}({1}^{{z}^*},{1}^{{w}^*},\theta^*,\lambda^*)-J({1}^{{z}^*},{1}^{{w}^*},\Phi^*)+J({1}^{{z}^*},{1}^{{w}^*},\Phi^*)-J(\hat{q}^{{z}},\hat{q}^{{w}},\hat{\Phi}) \right.  \\
	& \,\,  \left. +J(\hat{q}^{{z}},\hat{q}^{{w}},\hat{\Phi})-\bar{J}(\hat{q}^{{z}},\hat{q}^{{w}},\hat{\theta},\hat{\lambda})  \geq   C_1 mn   \epsilon \big \vert  {z}^*,{w}^*  \right ) \\
 \leq &  \,\, \mathbb{P} \Big(\bar{J}({1}^{{z}^*},{1}^{{w}^*},\theta^*,\lambda^*)-J({1}^{{z}^*},{1}^{{w}^*},\Phi^*) \\ 
 &\qquad + J(\hat{q}^{{z}},\hat{q}^{{w}},\hat{\Phi})-\bar{J}(\hat{q}^{{z}},\hat{q}^{{w}},\hat{\theta},\hat{\lambda}) \geq  C_1 mn  \epsilon   \big \vert  {z}^*,{w}^*\Big) \\
 \leq & \,\, \mathbb{P}\left (|\bar{J}({1}^{{z}^*},{1}^{{w}^*},\theta^*,\lambda^*)-J({1}^{{z}^*},{1}^{{w}^*},\Phi^*)|  \geq  (C_1/2) mn  \epsilon   \big \vert  {z}^*,{w}^*\right )\\
     & \,\, + \mathbb{P}\left (|J(\hat{q}^{{z}},\hat{q}^{{w}},\hat{\Phi})-\bar{J}(\hat{q}^{{z}},\hat{q}^{{w}},\hat{\theta},\hat{\lambda})|  \geq  (C_1/2) mn  \epsilon   \big \vert  {z}^*,{w}^*\right ) \rightarrow 0.
\end{align*}
The proof of $M_\textnormal{col}(\hat{q}^{{w}})=o_p(1)$ is similar.

\end{proof}

\section{Additional Simulation Results} \label{supp:simulations}

\subsection{Estimation Accuracy} \label{supp:est_accuracy}
We evaluate the estimation accuracy of the probability matrix under the primary simulation setting described in Section \ref{sec:simulation1}. For VEM, the estimated probability matrix $\hat{P}$ is calculated directly from the converged parameter estimates $\hat{\theta}_{il}$ and $\hat{\lambda}_{jk}$. For DOM and TSDC, we proceed as follows. Given estimated row and column labels, partition $A$ into $K \times L$ blocks $A^{(k, l)}(\hat{z}, \hat{w})$ and let $\hat{P}^{(k, l)} = \Pi_1(A^{(k, l)} (\hat{z}, \hat{w}))$ where $\Pi_1(X) = \arg \min_{\text{rank}(M) = 1} \| X - M \|_F$ denotes the best rank-one approximation of each block under the Frobenius norm $\|\cdot\|_F$. The full estimate $\hat{P}$ is assembled by filling each block accordingly. For DOM, this is precisely the estimator defined in \cite{jing2026two} (Equation (23));  for TSDC, which produces community labels but no explicit estimate of $P$, we apply the same block rank-one approximation post-hoc. Because this reconstruction depends on the quality of the estimated labels, MSE for TSDC and DOM reflects clustering and estimation error jointly. Estimation accuracy is measured by the normalized Frobenius error $MSE = ||P - \hat{P}||^2_F / (mn)$, which averages squared entry-wise errors across the full matrix and is comparable across network sizes. Figure \ref{fig:estimation_accuracy} reports the mean MSE ($\pm$ 2 SE) across 100 replicates. VEM achieves the lowest MSE across all three distributional setups and all values of $r$. Under the uniform setup the advantage over TSDC is modest, but the gap widens under the gamma and power law setups, where TSDC's error grows more steeply with $r$. DOM yields the highest MSE under most conditions and degrades most rapidly with network density, suggesting rank-one approximation may be inadequate as degree heterogeneity increases. The increase in MSE with $r$ across all methods is expected: denser networks produce larger absolute values of $P_{ij}$, inflating the squared errors even as relative recovery improves. 
\begin{figure}[H]
    \centering
    \includegraphics[width=1\linewidth]{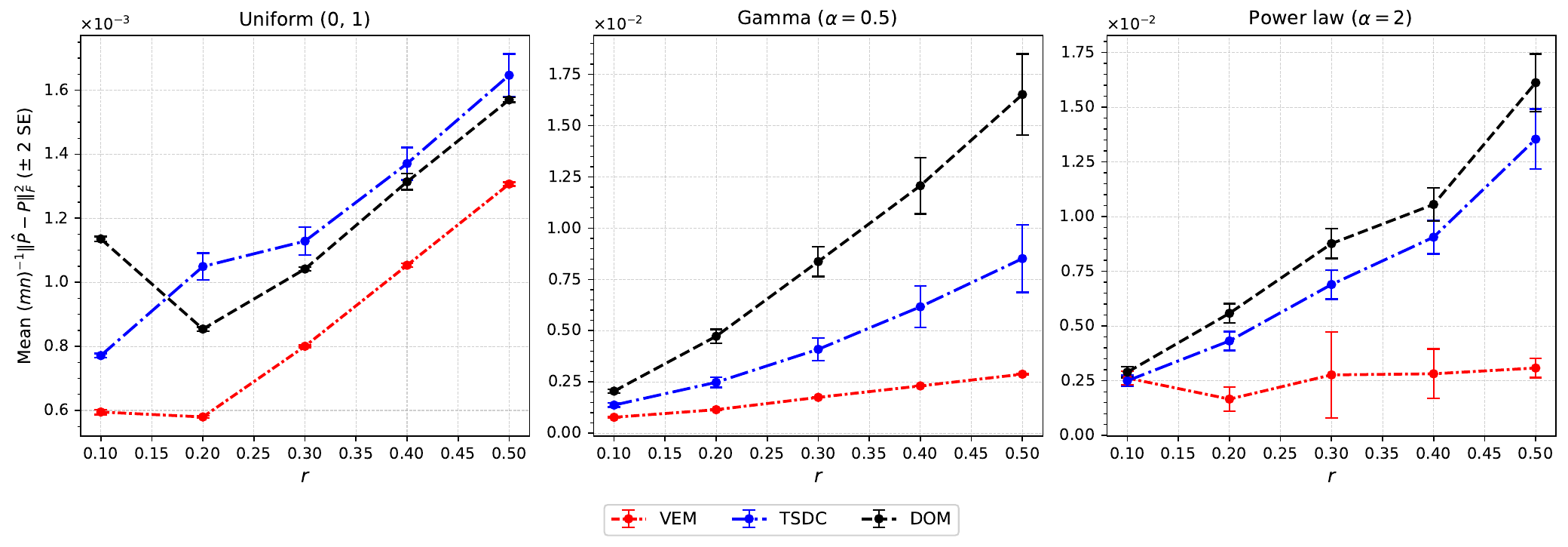}
    \caption{Mean squared error $||P - \hat{P}||^2_F / (mn)$ ($\pm$ 2 SE) for estimation of the probability matrix $P$ averaged over 100 replicates.}
    \label{fig:estimation_accuracy}
\end{figure}

\subsection{Additional Simulation Settings} \label{supp:robustness}
To evaluate robustness, we test three structural variations: (1) smaller networks ($m = 400, n = 500, K = 3, L=4)$; (2) networks with more communities ($m = 800, n = 1000, K=5, L=6)$; and (3) unbalanced networks ($m = 800, n=1000, K=3, L=4)$ with skewed row and column assignment probabilities of $\pi=(0.5, 0.3, 0.2)$ and $\rho=(0.4, 0.3, 0.2, 0.1)$, respectively. Across each variation, we vary the signal strength $r \in \{0.1, 0.3, 0.5\}$ and degree distributions (Uniform, Gamma, and Power Law). Table \ref{tab:supp_robustness} presents the mean Adjusted Rand Index (ARI) over 50 repetitions, comparing VEM against SVDK, TSDC, and DOM. In moderate to high signal regimes $(r \ge 0.3)$, VEM consistently achieves better community recovery across all tested distributions and network configurations. In the extreme low-signal regime $(r = 0.1)$, VEM maintains its dominance under uniform distributions and in unbalanced settings. However, under heavy-tailed degree distributions at $r = 0.1$, TSDC exhibits slightly higher ARI scores similar to the observation in the simulation section. Overall, VEM demonstrates strong robustness in the vast majority of settings.

\begin{table}[H]
\centering
\caption{Supplementary results. Each cell shows mean ARI $\pm$ std over 50 repetitions.}
\label{tab:supp_robustness}\par
\resizebox{\textwidth}{!}{%
\begin{tabular}{llcc|c||c|c||c|c||c|c}
\hline
\textbf{Setting} & \textbf{Distribution} & $r$ & \multicolumn{2}{c}{\textbf{SVDK}} & \multicolumn{2}{c}{\textbf{VEM}} & \multicolumn{2}{c}{\textbf{TSDC}} & \multicolumn{2}{c}{\textbf{DOM}} \\
\cmidrule(lr){4-5} \cmidrule(lr){6-7} \cmidrule(lr){8-9} \cmidrule(lr){10-11}
 & & & Row & Col & Row & Col & Row & Col & Row & Col \\
\hline
\multirow{9}{*}{\rotatebox{90}{\shortstack{Small \\ $400\times500$, $K=3,L=4$}}} & \multirow{3}{*}{Uniform} & 0.1 & 0.04 $\pm$ 0.03 & 0.02 $\pm$ 0.01 & \textbf{0.04 $\pm$ 0.03} & \textbf{0.02 $\pm$ 0.01} & 0.04 $\pm$ 0.03 & 0.02 $\pm$ 0.01 & 0.01 $\pm$ 0.01 & 0.02 $\pm$ 0.01 \\
 &  & 0.3 & 0.55 $\pm$ 0.17 & 0.15 $\pm$ 0.04 & \textbf{0.91 $\pm$ 0.03} & \textbf{0.80 $\pm$ 0.05} & 0.64 $\pm$ 0.18 & 0.24 $\pm$ 0.08 & 0.51 $\pm$ 0.23 & 0.39 $\pm$ 0.18 \\
 &  & 0.5 & 0.72 $\pm$ 0.18 & 0.19 $\pm$ 0.05 & \textbf{0.99 $\pm$ 0.01} & \textbf{0.96 $\pm$ 0.02} & 0.95 $\pm$ 0.05 & 0.79 $\pm$ 0.16 & 0.96 $\pm$ 0.01 & 0.90 $\pm$ 0.03 \\
\cmidrule(l){2-11}
 & \multirow{3}{*}{\shortstack[l]{Gamma \\ $\alpha=0.5$}} & 0.1 & 0.04 $\pm$ 0.04 & 0.02 $\pm$ 0.01 & 0.30 $\pm$ 0.15 & 0.23 $\pm$ 0.10 & \textbf{0.41 $\pm$ 0.21} & \textbf{0.31 $\pm$ 0.16} & 0.18 $\pm$ 0.14 & 0.21 $\pm$ 0.11 \\
 &  & 0.3 & 0.09 $\pm$ 0.07 & 0.03 $\pm$ 0.02 & \textbf{0.97 $\pm$ 0.01} & \textbf{0.92 $\pm$ 0.02} & 0.87 $\pm$ 0.17 & 0.79 $\pm$ 0.15 & 0.43 $\pm$ 0.28 & 0.45 $\pm$ 0.23 \\
 &  & 0.5 & 0.10 $\pm$ 0.07 & 0.04 $\pm$ 0.03 & \textbf{0.99 $\pm$ 0.01} & \textbf{0.97 $\pm$ 0.01} & 0.90 $\pm$ 0.20 & 0.88 $\pm$ 0.15 & 0.57 $\pm$ 0.28 & 0.56 $\pm$ 0.23 \\
\cmidrule(l){2-11}
 & \multirow{3}{*}{\shortstack[l]{Pareto \\ $\alpha=2$}} & 0.1 & 0.06 $\pm$ 0.07 & 0.03 $\pm$ 0.03 & 0.04 $\pm$ 0.04 & 0.03 $\pm$ 0.02 & \textbf{0.21 $\pm$ 0.13} & \textbf{0.13 $\pm$ 0.08} & 0.15 $\pm$ 0.11 & 0.13 $\pm$ 0.07 \\
 &  & 0.3 & 0.16 $\pm$ 0.16 & 0.10 $\pm$ 0.09 & 0.65 $\pm$ 0.21 & 0.54 $\pm$ 0.17 & \textbf{0.73 $\pm$ 0.21} & \textbf{0.60 $\pm$ 0.17} & 0.32 $\pm$ 0.16 & 0.35 $\pm$ 0.13 \\
 &  & 0.5 & 0.24 $\pm$ 0.25 & 0.14 $\pm$ 0.12 & \textbf{0.98 $\pm$ 0.02} & \textbf{0.95 $\pm$ 0.05} & 0.81 $\pm$ 0.23 & 0.78 $\pm$ 0.20 & 0.47 $\pm$ 0.24 & 0.47 $\pm$ 0.18 \\
\hline
\multirow{9}{*}{\rotatebox{90}{\shortstack{More communities \\ $800\times1000$, $K=5,L=6$}}} & \multirow{3}{*}{Uniform} & 0.1 & 0.10 $\pm$ 0.03 & 0.04 $\pm$ 0.01 & \textbf{0.11 $\pm$ 0.04} & \textbf{0.05 $\pm$ 0.01} & 0.10 $\pm$ 0.03 & 0.04 $\pm$ 0.01 & 0.01 $\pm$ 0.01 & 0.02 $\pm$ 0.01 \\
 &  & 0.3 & 0.68 $\pm$ 0.12 & 0.19 $\pm$ 0.06 & \textbf{0.99 $\pm$ 0.01} & \textbf{0.97 $\pm$ 0.01} & 0.95 $\pm$ 0.04 & 0.74 $\pm$ 0.16 & 0.96 $\pm$ 0.01 & 0.89 $\pm$ 0.07 \\
 &  & 0.5 & 0.74 $\pm$ 0.14 & 0.24 $\pm$ 0.07 & \textbf{1.00 $\pm$ 0.00} & \textbf{1.00 $\pm$ 0.00} & 0.99 $\pm$ 0.04 & 0.94 $\pm$ 0.09 & 1.00 $\pm$ 0.00 & 0.97 $\pm$ 0.06 \\
\cmidrule(l){2-11}
 & \multirow{3}{*}{\shortstack[l]{Gamma \\ $\alpha=0.5$}} & 0.1 & 0.09 $\pm$ 0.03 & 0.04 $\pm$ 0.02 & 0.41 $\pm$ 0.11 & 0.36 $\pm$ 0.13 & \textbf{0.85 $\pm$ 0.11} & \textbf{0.74 $\pm$ 0.11} & 0.30 $\pm$ 0.15 & 0.31 $\pm$ 0.13 \\
 &  & 0.3 & 0.16 $\pm$ 0.07 & 0.04 $\pm$ 0.02 & \textbf{1.00 $\pm$ 0.00} & \textbf{0.99 $\pm$ 0.00} & 0.97 $\pm$ 0.07 & 0.90 $\pm$ 0.11 & 0.59 $\pm$ 0.24 & 0.57 $\pm$ 0.23 \\
 &  & 0.5 & 0.16 $\pm$ 0.06 & 0.07 $\pm$ 0.03 & \textbf{1.00 $\pm$ 0.00} & \textbf{1.00 $\pm$ 0.00} & 0.99 $\pm$ 0.04 & 0.95 $\pm$ 0.09 & 0.76 $\pm$ 0.20 & 0.71 $\pm$ 0.21 \\
\cmidrule(l){2-11}
 & \multirow{3}{*}{\shortstack[l]{Pareto \\ $\alpha=2$}} & 0.1 & 0.07 $\pm$ 0.04 & 0.05 $\pm$ 0.03 & 0.03 $\pm$ 0.02 & 0.04 $\pm$ 0.03 & \textbf{0.36 $\pm$ 0.14} & \textbf{0.29 $\pm$ 0.12} & 0.21 $\pm$ 0.09 & 0.22 $\pm$ 0.09 \\
 &  & 0.3 & 0.17 $\pm$ 0.10 & 0.14 $\pm$ 0.08 & \textbf{0.87 $\pm$ 0.13} & \textbf{0.87 $\pm$ 0.11} & 0.81 $\pm$ 0.14 & 0.74 $\pm$ 0.15 & 0.48 $\pm$ 0.16 & 0.46 $\pm$ 0.12 \\
 &  & 0.5 & 0.26 $\pm$ 0.15 & 0.15 $\pm$ 0.10 & \textbf{0.99 $\pm$ 0.04} & \textbf{0.95 $\pm$ 0.08} & 0.85 $\pm$ 0.15 & 0.75 $\pm$ 0.17 & 0.57 $\pm$ 0.17 & 0.54 $\pm$ 0.18 \\
\hline
\multirow{9}{*}{\rotatebox{90}{\shortstack{Unbalanced \\ $800\times1000$, $K=3,L=4$}}} & \multirow{3}{*}{Uniform} & 0.1 & 0.22 $\pm$ 0.04 & 0.12 $\pm$ 0.03 & \textbf{0.28 $\pm$ 0.06} & \textbf{0.19 $\pm$ 0.07} & 0.22 $\pm$ 0.03 & 0.12 $\pm$ 0.04 & 0.07 $\pm$ 0.04 & 0.11 $\pm$ 0.04 \\
 &  & 0.3 & 0.35 $\pm$ 0.05 & 0.34 $\pm$ 0.09 & \textbf{0.99 $\pm$ 0.00} & \textbf{0.98 $\pm$ 0.01} & 0.53 $\pm$ 0.18 & 0.54 $\pm$ 0.17 & 0.89 $\pm$ 0.19 & 0.85 $\pm$ 0.14 \\
 &  & 0.5 & 0.36 $\pm$ 0.03 & 0.37 $\pm$ 0.09 & \textbf{1.00 $\pm$ 0.00} & \textbf{1.00 $\pm$ 0.00} & 0.75 $\pm$ 0.25 & 0.82 $\pm$ 0.20 & 0.90 $\pm$ 0.20 & 0.95 $\pm$ 0.10 \\
\cmidrule(l){2-11}
 & \multirow{3}{*}{\shortstack[l]{Gamma \\ $\alpha=0.5$}} & 0.1 & 0.04 $\pm$ 0.07 & 0.02 $\pm$ 0.03 & \textbf{0.87 $\pm$ 0.07} & \textbf{0.76 $\pm$ 0.08} & 0.64 $\pm$ 0.22 & 0.62 $\pm$ 0.13 & 0.41 $\pm$ 0.23 & 0.43 $\pm$ 0.16 \\
 &  & 0.3 & 0.05 $\pm$ 0.06 & 0.02 $\pm$ 0.03 & \textbf{0.99 $\pm$ 0.01} & \textbf{0.97 $\pm$ 0.01} & 0.84 $\pm$ 0.23 & 0.84 $\pm$ 0.15 & 0.60 $\pm$ 0.28 & 0.64 $\pm$ 0.23 \\
 &  & 0.5 & 0.05 $\pm$ 0.08 & 0.02 $\pm$ 0.03 & \textbf{1.00 $\pm$ 0.00} & \textbf{0.99 $\pm$ 0.01} & 0.80 $\pm$ 0.25 & 0.86 $\pm$ 0.14 & 0.66 $\pm$ 0.28 & 0.69 $\pm$ 0.24 \\
\cmidrule(l){2-11}
 & \multirow{3}{*}{\shortstack[l]{Pareto \\ $\alpha=2$}} & 0.1 & 0.06 $\pm$ 0.10 & 0.05 $\pm$ 0.06 & 0.22 $\pm$ 0.14 & 0.18 $\pm$ 0.11 & \textbf{0.55 $\pm$ 0.21} & \textbf{0.43 $\pm$ 0.16} & 0.27 $\pm$ 0.15 & 0.29 $\pm$ 0.13 \\
 &  & 0.3 & 0.27 $\pm$ 0.29 & 0.23 $\pm$ 0.21 & \textbf{1.00 $\pm$ 0.01} & \textbf{0.96 $\pm$ 0.07} & 0.84 $\pm$ 0.20 & 0.81 $\pm$ 0.18 & 0.58 $\pm$ 0.26 & 0.58 $\pm$ 0.21 \\
 &  & 0.5 & 0.40 $\pm$ 0.32 & 0.39 $\pm$ 0.26 & \textbf{1.00 $\pm$ 0.00} & \textbf{0.98 $\pm$ 0.06} & 0.90 $\pm$ 0.17 & 0.86 $\pm$ 0.16 & 0.65 $\pm$ 0.25 & 0.70 $\pm$ 0.20 \\
\hline
\end{tabular}%
}
\smallskip
\end{table}

\end{document}